\documentclass[lettersize,journal]{IEEEtran}
\usepackage{amsmath,amsfonts}
\usepackage{algorithmic}
\usepackage{algorithm}
\usepackage{array}
\usepackage[caption=false,font=normalsize,labelfont=sf,textfont=sf]{subfig}
\usepackage{textcomp}
\usepackage{stfloats}
\usepackage{url}
\usepackage{verbatim}
\usepackage{graphicx}
\usepackage{cite}
\usepackage{amsthm}
\usepackage{booktabs}
\usepackage{amssymb}   % for \checkmark
\usepackage{pifont}
\usepackage[export]{adjustbox}
\usepackage{subcaption}
\usepackage{booktabs}
\usepackage{multirow}
\usepackage{thmtools}
\usepackage{thm-restate}
\usepackage{subfig}
\usepackage{amsmath}
\usepackage{array}

\begin{document}

\title{Causality-Driven Disentangled Representation Learning \\ in Multiplex Graphs}

\author{Saba Nasiri, Selin Aviyente, Dorina Thanou 
        % <-this % stops a space
%\thanks{This paper was produced by the IEEE Publication Technology Group. They are in Piscataway, NJ.}% <-this % stops a space
\thanks{S. Nasiri and D. Thanou are with EPFL, Lausanne, Switzerland
(e-mail: saba.nasiri@epfl.ch; dorina.thanou@epfl.ch). S. Aviyente is with the Department of Electrical and Computer Engineering at Michigan State University, East Lansing, MI (e-mail: aviyente@egr.msu.edu). This work was supported in part by  SNSF under the Lead Agency grant 229452 and NSF under ECCS-2430516. }

\thanks{Code is publicly available at: 
\protect\url{https://github.com/SabaNasirii/CaDeM}}

}
%,~\IEEEmembership{Staff,~IEEE,}

% The paper headers
\markboth{}%
{Shell \MakeLowercase{\textit{et al.}}: A Sample Article Using IEEEtran.cls for IEEE Journals}

\IEEEpubid{}
% Remember, if you use this you must call \IEEEpubidadjcol in the second
% column for its text to clear the IEEEpubid mark.

\maketitle

\begin{abstract}
Learning representations from multiplex graphs, i.e., multi-layer networks where nodes interact through multiple relation types, is challenging due to the entanglement of shared (common) and layer-specific (private) information, which limits generalization and interpretability. In this work, we introduce a causal inference–based framework that disentangles common and private components in a self-supervised manner. CaDeM jointly (i) aligns shared embeddings across layers, (ii) enforces private embeddings to capture layer-specific signals, and (iii) applies backdoor adjustment to ensure that the common embeddings capture only global information while being separated from the private representations. Experiments on synthetic and real-world datasets demonstrate consistent improvements over existing baselines, highlighting the effectiveness of our approach for robust and interpretable multiplex graph representation learning.
\end{abstract}


\section{Introduction}
\label{sec:intro}

Multiplex graphs are multi-layer networks where the same set of nodes is connected through different types of edges, each representing a distinct relation \cite{Kivela_2014}. 
This structure is particularly valuable in complex, multi-modal domains where data is naturally structured across multiple layers. For instance, brain activity can be modeled using multiple graph views, where one layer reflects structural connectivity, and another captures functional connectivity, providing complementary perspectives on brain function. Similarly, social networks contain different interaction layers such as ``follow'', ``like'', or messaging. Disentangling representations across these views facilitates the recovery of the underlying generative factors driving interactions, allowing models to distinguish global patterns from interaction-specific dynamics in applications such as recommendation and community detection.

However, learning effective representations in multiplex graphs remains challenging. In many settings, node behavior is governed by both common embeddings, which capture characteristics stable across layers, and private embeddings, which encode layer-specific signals.
Most existing methods entangle these components by projecting them into a single embedding space \cite{Xiaowen_Dong_2012}, \cite{Liu2017PrincipledMN}, \cite{DEFORD2019121949}. Such aggregation often obscures global patterns and corrupts layer-specific features, leading to reduced robustness and limited interpretability. Disentanglement addresses this issue by explicitly separating common and private embeddings so that each captures only the information relevant to its intended role. 
Yet, recent approaches in the context of multiplex graphs often rely on contrastive or orthogonality objectives as proxies for disentanglement \cite{NEURIPS2021_b6cda17a}, \cite{ZHANG2024110839}, \cite{pmlr-v202-mo23a}. While effective in practice, such objectives do not explicitly control for confounding effects introduced by shared signals that may spuriously correlate with layer-dependent targets. This provides no explicit mechanism to prevent layer-specific information from influencing the common branch or to prevent the private branches from relying on shared shortcuts.

We introduce \textbf{CaDeM}, a \textbf{Ca}usality-driven self-supervised framework for \textbf{D}isentangled representation learning in \textbf{M}ultiplex graphs. CaDeM builds on a Graph Convolutional Neural Network (GCN) backbone and combines three complementary objectives: (i) a matching objective that aligns common embeddings across layers, (ii) a self-supervised objective that drives private embeddings to capture meaningful layer-specific signals while avoiding trivial solutions, and (iii) a causal loss that approximates backdoor adjustment via confounder stratification and re-pairing of private and common embeddings. Together, these objectives enforce disentanglement by ensuring that common embeddings retain information shared across layers, while private embeddings capture layer-specific signals. 

Extensive experiments on both synthetic multiplex graphs and real-world datasets spanning citation networks, recommendation systems, and brain networks demonstrate the effectiveness of CaDeM. In particular, it consistently achieves the best Macro- and Micro-F1 scores across application domains on both node-level and graph-level tasks. 
Qualitative and quantitative results confirm that CaDeM disentangles shared and layer-specific components, yielding representations that capture interpretable factors aligned with their intended roles.
Overall, by reducing confounding, CaDeM produces cleaner representations, leading to improved performance in downstream tasks.

\section{Related work}
\label{sec:relwork}
In this section, we review three key strands of research that motivate our approach. 

\textbf{Self-supervised learning on graphs.} Self-supervised learning on graphs, using techniques such as node masking, edge prediction, or contrastive objectives, has become a dominant paradigm for single-layer graphs \cite{9770382}. Extending this idea to multiplex graphs, early methods often combined information from multiple layers in simple ways, for example, by concatenating per-layer embeddings or applying joint matrix factorization \cite{Xiaowen_Dong_2012}, \cite{Liu2017PrincipledMN}, \cite{DEFORD2019121949}. 
A recent study \cite{10.1007/978-3-031-63778-0_1} analyzes fusion strategies in multiplex graphs and introduces different levels of fusion, showing that very early (graph-level) or very late (prediction-level) fusion tends to discard important layer-specific or cross-layer information.
HDMI \cite{Jing_2021} is a self-supervised framework that maximizes mutual information to capture both node–global and node–attribute dependencies. An attention-based fusion module unifies per-layer embeddings, enabling HDMI to achieve strong performance on different tasks. DMGI \cite{park2020unsupervisedattributedmultiplexnetwork} extends DGI \cite{veličković2018deep} to multiplex graphs by learning relation-type–specific encoders and using a consensus regularization framework with a shared discriminator to maximize mutual information between local node embeddings and global graph summaries across relation types. More recently, Peng \textit{et al.} proposed CoCoMG \cite{Peng_Wang_Zhu_2023}, which encourages complementary and consistent information across graph layers using correlation-based objectives. It uses a local structure preservation loss and a CCA-based objective to enforce cross-layer consistency. HMGE \cite{abdous2024hierarchical} introduces a hierarchical aggregation strategy for high-dimensional multiplex graphs and trains the model through mutual information maximization between local patches and global summaries. However, these approaches do not explicitly disentangle shared and layer-specific components.

\textbf{Disentangled representation learning on graphs. } Parallel work emphasizes disentangled representations, particularly in self-supervised settings. DGCL applies contrastive objectives to separate latent dimensions in graph-level embeddings \cite{NEURIPS2021_b6cda17a},  while DEGREE disentangles shared and view-specific representations using orthogonality constraints to improve multi-view clustering  \cite{ZHANG2024110839}. More recently, DMGRL further extracts common and private components through separate graph neural paths and enforces desired properties using orthogonality, reconstruction, and contrastive constraints \cite{pmlr-v202-mo23a}. 
More recently, Hu \textit{et al.} \cite{hu2024disentangledgenerativegraphrepresentation} introduce DiGGR, a generative framework that learns disentangled latent factors to guide graph masking and reconstruction, improving both robustness and interpretability.
These studies highlight the importance of disentanglement, showing that enforcing factor independence leads to more generalizable embeddings.

\textbf{Causal inference in graph learning.}
Causal tools for robust graph representation learning have recently gained attention. Standard graph neural networks (GNNs) often conflate causally relevant features with confounders, leading to spurious correlations in downstream tasks. Recent work incorporates causal principles to mitigate this issue.  CAL \cite{Sui_2022} explicitly
parameterizes the backdoor adjustment formula to separate causal and shortcut features and enforces invariant relationships between causal patterns and model predictions, while Casper \cite{Jing_2024} extends causal attention to spatiotemporal GNNs via front-door adjustment. To further mitigate confounding effects, Gao \textit{et al.} \cite{10.1609/aaai.v37i6.25925} propose RCGRL, which actively generates instrumental variables to remove confounders within GNNs, enabling the model to learn representations that depend only on causal features. While promising, these methods do not consider multiplex structures, where shared and layer-specific information naturally coexist. Our work is, to the best of our knowledge, the first to operationalize causal adjustment for disentanglement in multiplex graphs by explicitly stratifying and re-pairing private (causal) and common (confounding) embeddings in different layers to isolate causal effects in downstream tasks.
This approach enables a clearer separation between common and private embeddings for their specific roles.

\section{{C\MakeLowercase{a}D\MakeLowercase{e}M}: Causality-Driven Disentanglement in Multiplex Graphs}
\label{sec:method}
Motivated by these limitations, we introduce \textbf{CaDeM}, a self-supervised framework for causality-driven disentangled representation learning in multiplex graphs. CaDeM learns common embeddings capturing information shared across layers and private embeddings encoding layer-specific features. The framework builds on a Graph Convolutional Network (GCN) backbone and integrates three complementary objectives: (i) aligning common embeddings across layers, (ii) encouraging private embeddings to capture layer-specific information, and (iii) incorporating a causality-inspired mechanism that ensures common embeddings remain disentangled from private representations.

\begin{figure}[!t]
  \centering

  \includegraphics[width=\columnwidth]{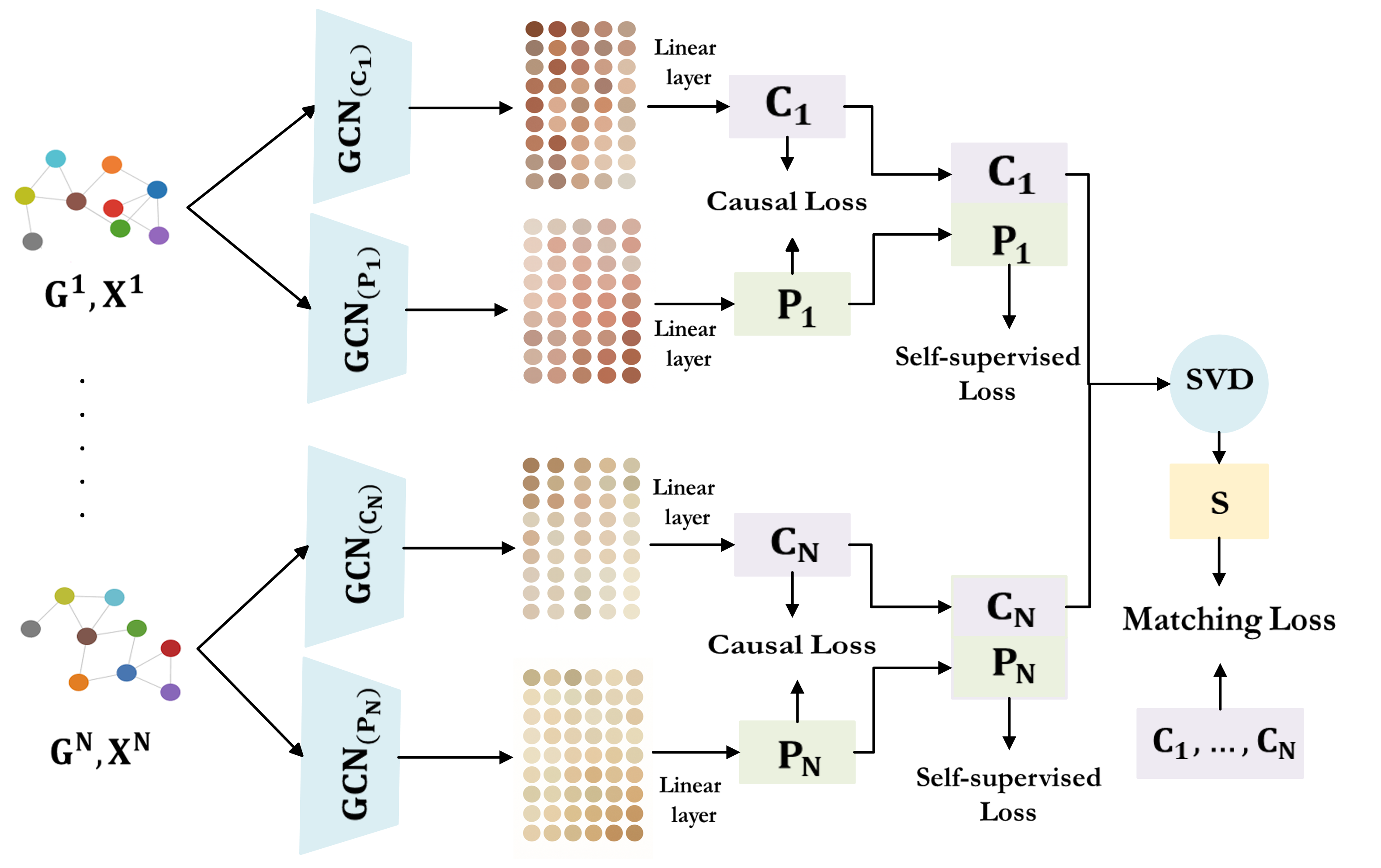}\\[-0.5ex]
  {\footnotesize (a) CaDeM End-to-End Pipeline Architecture}
  \vspace{0.3cm}

  \includegraphics[width=\columnwidth]{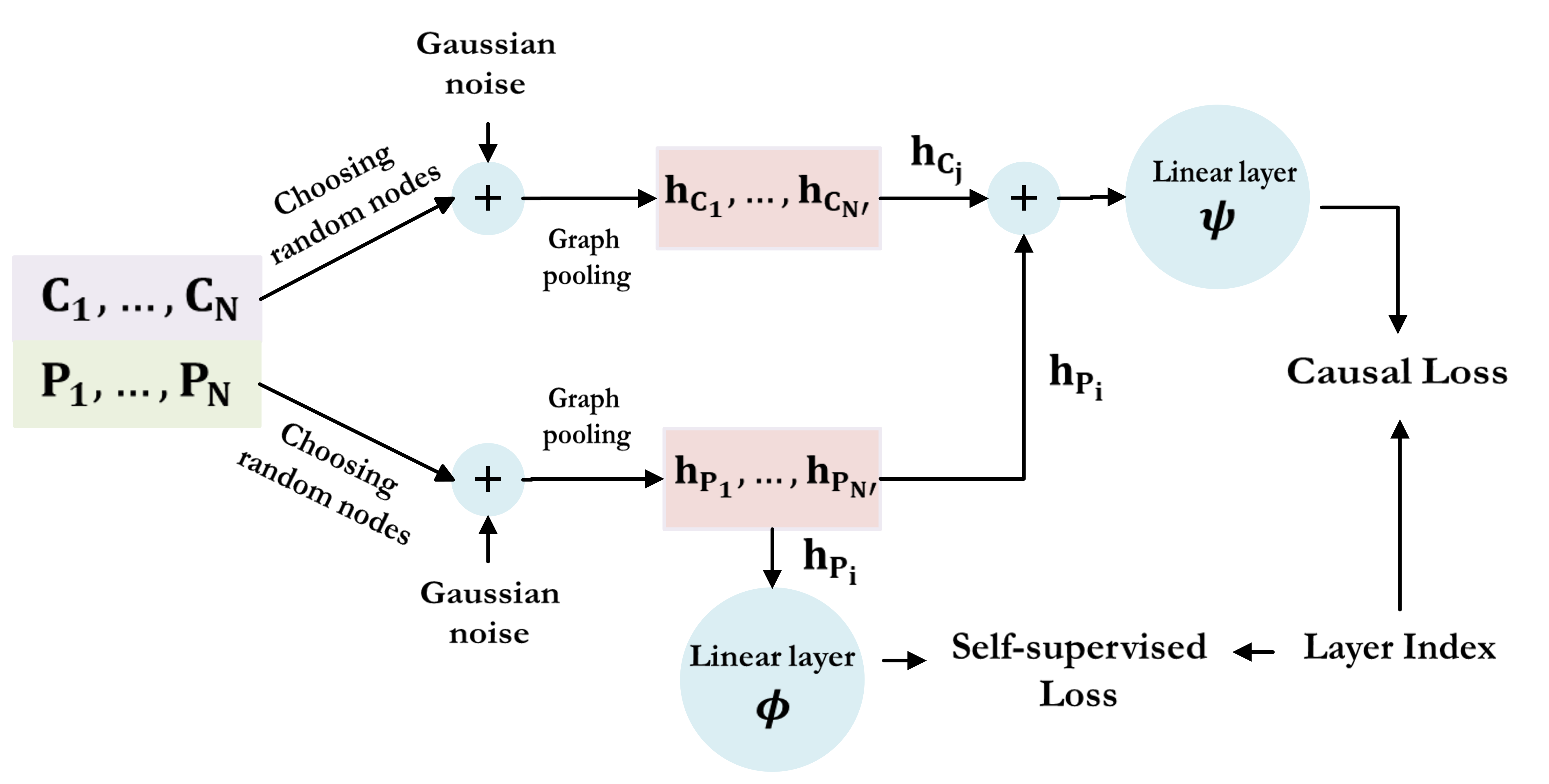}\\[-0.5ex]
  {\footnotesize (b) Network architecture for graph augmentation used to support self-supervised and causal losses in CaDeM.}
  \caption{CaDeM Network architecture for extracting common and private embeddings, with losses enforcing desired characteristics.}
  \label{fig:two_images}
\end{figure}

\subsection{ Problem formulation and framework overview}
We consider a multiplex graph $G = \{ G^{1}, \ldots, G^{N} \}$, 
where $N$ denotes the number of layers. Each layer  $G^{\ell} = (V, E^{\ell}, X^{\ell})$ shares a common node set $V = \{ v_{1}, \ldots, v_{M} \}$, with $M$ nodes, but has distinct edge set $E^{\ell}$ and corresponding node feature matrix $X^{\ell} \in \mathbb{R}^{M \times F}$, where $F$ is the input feature dimension. Given $G$, CaDeM learns two types of embeddings for each layer: common embeddings $C_{1}, \ldots,  C_{N} \in \mathbb{R}^{M \times d}$ which capture information shared across layers,  and private embeddings $P_{1}, \ldots, P_{N} \in \mathbb{R}^{M \times d}$, which encode layer-specific features.  
Here, $d$ denotes the embedding dimension. For the theoretical analysis that follows, we also introduce a probabilistic view over the layers. In particular, let
$\mathbf{L}$ denotes the discrete random variable representing the layer index, distributed uniformly over $\{1,\dots, N\}$. We denote by $\mathbf{P}_{\mathbf{L}}$ and $\mathbf{C}_{\mathbf{L}}$ the private and common embedding random variables corresponding to the layer indexed by the random variable $\mathbf{L}$. To obtain these representations, we employ two parallel GCN encoders per layer: one dedicated to learning common embeddings (e.g., $\text{GCN}_{(C_1)}, \ldots, \text{GCN}_{(C_N)}$) and the other to learning private embeddings (e.g., $\text{GCN}_{(P_1)}, \ldots, \text{GCN}_{(P_N)}$). Each GCN consists of learnable weight matrices that aggregate information from node neighbors and transform into latent representations \cite{kipf2017semisupervised}. The outputs of the encoders are further projected using linear layers into the final embedding spaces for $C_\ell$ and $P_\ell$. Training is guided by three objectives, i.e., matching loss, self-supervised loss, and causal loss, which jointly enforce alignment of common embeddings across layers, capture layer-specific signals through private embeddings, and disentangle common and private embeddings for their intended role, as described in the following subsections. 
Fig.~\ref{fig:two_images} illustrates the overall architecture.

\subsection{Aligning common embeddings}
The matching loss aligns common embeddings across layers by encouraging representations of the same node to map to a shared latent space. Intuitively, it promotes a
common component that captures invariant information across layers while preventing leakage of layer-specific signals into the shared embedding. To obtain such a unified representation, we aggregate the common embeddings across
layers as $C' = \sum_{\ell=1}^{N} C_\ell \in \mathbb{R}^{M \times d}$, 
and seek a matrix $S \in \mathbb{R}^{M \times d}$ that best represents the shared structure. This leads naturally to the framework of Generalized Procrustes Analysis (GPA) \cite{Gower_1975}, which finds a consensus configuration minimizing the total discrepancy between $S$ and the individual embeddings $C_\ell$:
\begin{equation}
\min_{S} \sum_{\ell=1}^{N} \|C_\ell - S\|_F^2 
\quad \text{s.t.} \quad S^{T} S = I. 
\end{equation}
The classical solution \cite{https://doi.org/10.1111/j.1745-3984.2003.tb01108.x} computes the singular value decomposition of the aggregated matrix,  $C' = U \Sigma V^{T}$, and defines the shared embedding as $S = U V^{T}$.
This yields the orthonormal configuration closest (in Frobenius norm) to the combined embeddings and serves as the shared embedding across layers. Following standard practices in embedding alignment, we additionally impose a zero-mean constraint on $S$ to eliminate global shifts and ensure identifiability. The matching loss is therefore defined as:
\begin{equation}
\begin{aligned}
\mathcal{L}_{\text{matching}}
&= \sum_{\ell=1}^{N} \| C_\ell - S \|_{F}^{2}, 
\hspace{4pt} \text{s.t.}\hspace{4pt} 
S^{\top} S = I \in \mathbb{R}^{d \times d},\hspace{3.5pt}
S^{\top} \mathbf{1} = \mathbf{0},
\end{aligned}
\label{eq:matching}
\end{equation}
 where $ \mathbf{1} \in \mathbb{R}^{M}, \mathbf{0} \in \mathbb{R}^{d}$. These constraints ensure that $S$ provides a unified representation that preserves shared information while aligning embeddings across layers.

\subsection{Causality-aware disentanglement of graph embeddings}
In graph representation learning, predictive features can be decomposed into causal components, which reflect stable, task-relevant mechanisms, and confounders, which induce spurious correlations by opening backdoor paths between the input graph and downstream tasks \cite{10.1609/aaai.v37i6.25925}. We generalize this perspective to multiplex graphs to guide the learning of private and common embeddings. 

\emph{Self-supervised formulation.} We define the downstream task in a self-supervised manner: each graph is assigned a layer-specific target capturing information unique to that layer, such as average degree, clustering coefficient, or other layer-specific properties. For simplicity and generality, we choose the layer index as the target, and train the model to predict it from the learned embeddings. In this formulation, private embeddings act as causal features because they encode the layer-dependent information necessary for prediction, while common embeddings act as confounders since they capture signals shared across layers that correlate with both the input graphs and the prediction target, potentially introducing spurious shortcuts. 

\emph{Backdoor adjustment.} The causal relationships among these components are illustrated in Fig. \ref{figdiagram}. This diagram reveals a backdoor path between the private embeddings $P_\ell$ and target prediction, i.e., $P_\ell \leftarrow \{G^\ell, X^\ell\} \rightarrow C_\ell \rightarrow G^\ell$ Embeddings $\rightarrow \text{target}$. Along this path, $C_\ell$ acts as a confounder between $P_\ell$ and the target prediction, creating a backdoor path that must be controlled during training.
\begin{figure}[!!!!!t]
  \centering
  \includegraphics[width=0.65\linewidth]{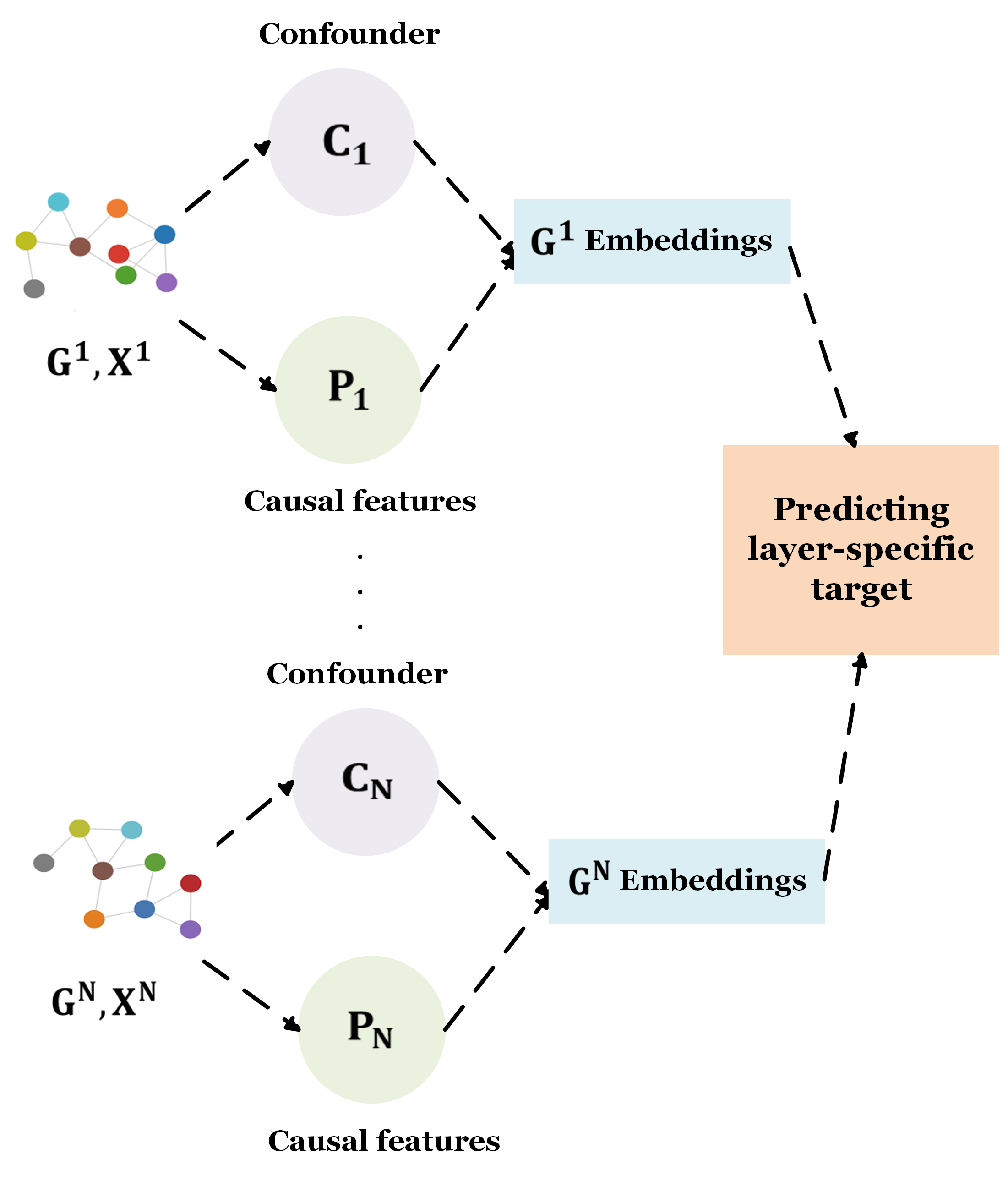}
  \caption{Causal diagram illustrating the roles of private and common embeddings in the layer-specific prediction task.}
  \label{figdiagram}
\end{figure}

To isolate the causal contribution of private embeddings (causal features) and mitigate the bias introduced by common embeddings (confounders), we adopt the backdoor adjustment principle from causal inference. Specifically, we block the backdoor path by conditioning on the common embeddings. Eq.~\eqref{backdoor} presents the backdoor adjustment formula adapted to our multiplex graph setting, where $\mathcal{C}$ denotes the set of common embeddings:

{\fontsize{8}{8}\selectfont
\begin{equation}
\begin{aligned}
&p(\text{task} \mid \mathrm{do}(\mathbf{P_L})) 
\overset{\text{Bayes Rule}}{=} \sum_{C_\ell \in \mathcal{C}} p(\text{task} \mid \mathrm{do}(\mathbf{P_L}), \mathbf{C_L} = C_\ell)\, p(\mathbf{C_L} = C_\ell)\\[4pt] 
&= \sum_{C_\ell \in \mathcal{C}} p(\text{task} \mid \mathbf{P_L}, \mathbf{C_L} = C_\ell)\, p(\mathbf{C_L} = C_\ell)
\quad \text{(Backdoor Criterion)}.
\end{aligned}
\label{backdoor}
\end{equation}
}
The backdoor criterion \cite{Pearl2014CausalMediation} states that if all variables that confound the relationship between a cause and its effect are properly conditioned on, the true causal effect can be identified. In our setting, this translates to estimating the effect of private embeddings on the downstream task while adjusting for the influence of common embeddings, which act as confounders. Direct implementation of backdoor adjustment on graphs is challenging
\cite{Sui_2022}. In particular, we approximate backdoor adjustment through a confounder stratification strategy. Specifically, each private embedding is re-paired with multiple realizations of common embeddings, enforcing prediction invariance across all pairings. This encourages the model to base predictions solely on private embeddings while neutralizing the spurious influence of common embeddings, thus disentangling causal and confounding signals and yielding more robust and interpretable embeddings.

\emph{Graph augmentation for stratification.} 
Since multiplex graphs often have only a small number of layers, we generate augmented graphs to obtain richer realizations of private and common components for stratification and more reliable loss estimation. Specifically, we create additional graphs by randomly sampling a ratio $r$ of the nodes from each layer, thereby augmenting the total number of graphs from $N$ to $N'$. For each augmented graph, we perturb node-level embeddings by injecting Gaussian noise with zero mean and a relatively small standard deviation $\sigma$ into both common and private embeddings. The parameters $r$, $\sigma$, and $N'$ are selected via hyperparameter tuning as described in Section \ref{eval_setup}. 

We define $\tau = \{ \tilde{G}_1, \ldots, \tilde{G}_{N'}\}$ as the set of augmented single-layer graphs. The layer-specific target $y_{\tilde{G}_i} \in \mathbb{R}^N$ is defined as the one-hot encoded layer index of graph $\tilde{G}_i \in \tau$, which is used to differentiate between graphs across layers. An add-pooling operation aggregates node-level embeddings into graph-level representations, denoted as $h_{C_i} \in \mathbb{R}^{d}$ for the common component and $h_{P_i} \in \mathbb{R}^{d}$ for the private component of the graph $\tilde{G}_i \in \tau$.
Predictions are then obtained through trainable networks $\phi$ and $\psi$, implemented as linear layers followed by a softmax activation, which operate on the private and stratified graph embeddings, respectively. Fig. \ref{fig:two_images}(b) illustrates the network architecture that separates causal (private) from confounding (common) embeddings within a causal inference framework, forming the basis for our self-supervised and causal losses.

\emph{Self-supervised and causal losses.} Formally, the self-supervised loss is formulated as a cross-entropy loss that trains private embeddings to predict the layer index:
\begin{equation}
\begin{aligned}
\mathcal{L}_{\text{self-supervised}} = -\frac{1}{N'}\sum_{\tilde{G}_i\in \tau} y_{\tilde{G}_i}^T \log \phi(h_{P_i}).
\end{aligned}
\label{eq2}
\end{equation}
Because private embeddings are designed to capture layer-specific characteristics, they should be predictive of the assigned targets. This loss serves as a regularizer, preventing private embeddings from collapsing into trivial or overly similar representations. Moreover, minimizing this objective maximizes the mutual information between private embeddings and the layer index, ensuring that the private embeddings retain the information required to predict layer-specific targets. This result is formalized in Proposition \ref{theorem1}.
\newtheorem{theorem}{Theorem}
\newtheorem{proposition}{Proposition}
\begin{proposition}
\label{theorem1}
Assuming the prediction head $\phi$ is sufficiently expressive to represent the Bayes posterior $p(\mathbf{L} \mid \mathbf{P_L})$, minimizing  $\mathcal{L}_{\text{self-supervised}}$ is equivalent to maximizing the mutual information $I(\mathbf{P_L}; \mathbf{L})$ between the private embeddings $\mathbf{P_L}$ and the layer index $\mathbf{L}$.
\end{proposition}
The proof is provided in the Supplementary Material.

We define the causal loss as a practical realization of the backdoor adjustment:
\begin{equation}
  \begin{aligned}
  \mathcal{L}_{\textit{causal}}
    = -\frac{1}{N'^2}\sum_{\tilde{G}_i\in \tau}\sum_{\tilde{G}_j\in \tau} 
      y_{\tilde{G}_i}^T \log \psi(h_{P_i} \oplus h_{C_{j}}), 
\end{aligned}
\label{eq3}
\end{equation}
where for each graph $\tilde{G}_i \in \tau$, its private embedding $h_{P_i}$ (causal component) is paired with the common embedding $h_{C_{j}}$ (confounder) of a another  graph $\tilde{G}_j \in \tau$. The operator $\oplus$ combines $h_{P_i}$ and $h_{C_j}$ (e.g., via concatenation or addition). This stratified pairing approximates backdoor adjustment by exposing the predictor to multiple realizations of the confounder for the same causal component. Together, the causal and matching losses encourage the removal of layer-specific information from the common embeddings.
In particular, optimizing these losses drives the layer-index predictions derived from the common embeddings toward a uniform distribution, indicating that the confounding influence of common features has been removed. This property is formalized in Proposition \ref{Theorem2}, whose complete proof is provided in the Supplementary Material.
\begin{proposition}
\label{Theorem2}
At optimality, minimizing $\mathcal{L}_{\textit{matching}}$ enforces 
that the predicted distribution of the layer index $\mathbf{L}$ based solely on the common embedding becomes uniform, under the assumption that 
$\mathbf{L} \sim \mathrm{Unif}\{1,\dots,N\}$. Moreover, assuming the predictor $\psi$ is sufficiently expressive, the stratified pairing in the causal loss enforces conditional independence between the common embedding and the layer index given the private embedding. This implies that, given the private embedding, the common embedding provides no additional information for predicting $\mathbf{L}$.
\end{proposition}
In practice, we observe empirical evidence consistent with this result. Specifically, as training progresses, the predicted layer-index distributions based on the common embeddings converge toward a uniform distribution.

Finally, we optimize the combined objective:
\begin{equation}
\begin{aligned}
\mathcal{L}_{\textit{final}} 
= \mathcal{L}_{\textit{matching}} 
+ \alpha \mathcal{L}_{\textit{self-supervised}} 
+ \beta \mathcal{L}_{\textit{causal}},
\end{aligned}
\label{eq4}
\end{equation}
where $\alpha$ and $\beta$ are hyperparameters that control the relative contributions of the self-supervised and causal loss terms, respectively, and are selected via hyperparameter tuning. 

Further analysis of CaDeM and its computational complexity is provided in the Supplementary Material.

\emph{Theoretical justification of disentanglement.} The integration of $\mathcal{L}_{\textit{matching}}$, $\mathcal{L}_{\textit{self-supervised}}$, and $\mathcal{L}_{\textit{causal}}$ jointly drives the disentanglement of common and private embeddings in multiplex graphs. We establish the following theoretical guarantee on the disentanglement of the learned embeddings. A detailed proof is provided in the Supplementary Material.
\begin{theorem}
\label{theorem3}
Assume a latent-factor model in which the observed data of a randomly drawn layer $\mathbf{L}$ of a multiplex graph is generated as
\begin{equation}
\mathbf{Z}_{\mathbf{L}}
=
q_{\mathbf{L}}(\mathbf{U}, \mathbf{V}_{\mathbf{L}}, \boldsymbol{\varepsilon}_{\mathbf{L}}),
\end{equation}
where $\mathbf{U}$ is a shared latent factor random variable,
$\mathbf{V}_{\mathbf{L}}$ is a layer-specific latent factor random variable,
$q_{\mathbf{L}}$ is a measurable function,
and $\boldsymbol{\varepsilon}_{\mathbf{L}}$ is a noise random variable. Assume that the layer-specific latent factors and noise terms are mutually independent across layers, and that the shared factor $\mathbf{U}$ is independent of each layer-specific factor. Let $f_C$ and $f_P$ denote measurable encoder functions mapping
$\mathbf{Z}_{\mathbf{L}}$ to the common and private embeddings:
\begin{equation}
\label{fp-fc}
\mathbf{C}_{\mathbf{L}} = f_C(\mathbf{Z}_{\mathbf{L}}),
\qquad
\mathbf{P}_{\mathbf{L}} = f_P(\mathbf{Z}_{\mathbf{L}}).
\end{equation}
Assume that the function classes for $f_C$ and $f_P$
are sufficiently rich to approximate any square-integrable measurable function of their inputs, and that all loss components in
Eq.~\eqref{eq4} attain their global minima. Then, at any global optimum:

\begin{enumerate}
\item[(i)] $\mathbf{C}_{\mathbf{L}}$ is almost surely a measurable function of $\mathbf{U}$ alone;

\item[(ii)] $\mathbf{P}_{\mathbf{L}}$ is almost surely a measurable function of $\mathbf{V}_{\mathbf{L}}$ alone;

\item[(iii)] $\mathbf{C_L}$ and $\mathbf{P_L}$ are disentangled in the sense that each encodes only role-specific information.
\end{enumerate}

\end{theorem}
As a result, common embeddings depend only on the shared latent factor, 
 while private embeddings depend only on the layer-specific latent factor.  
This separation arises from the joint effect of our three loss components: the matching loss eliminates layer-dependent variability from $\mathbf{C_L}$, the self-supervised loss encourages $\mathbf{P_L}$ to capture layer-specific information necessary to identify the layer, while the causal loss suppresses the predictive influence of common embeddings, which otherwise act as confounders 
and prevents layer-specific signals from leaking into $\mathbf{C_L}$. 
These mechanisms align
the common and private embeddings 
with the underlying shared and layer-specific latent factors, yielding disentangled representations in which each component captures only the information required for its intended role.

\section{Experiments}
\label{sec:exp}

\subsection{Evaluation setup}
\label{eval_setup}
We present results on four synthetic and five real-world datasets. Synthetic experiments enable controlled evaluation, as the shared and layer-specific factors are explicitly defined, allowing us to assess whether common and private embeddings capture the intended information.
We further evaluate CaDeM on five real-world datasets spanning citation networks (ACM, DBLP) \cite{Jing_2021}, movie-related datasets (IMDB \cite{Jing_2021}, Freebase \cite{pmlr-v202-mo23a}), and a neuroscience dataset from the Human Connectome Project (HCP) \cite{VanEssen2013HCP}, demonstrating strong performance across diverse domains.

\emph{Baselines.} For comparison, we include six single-view methods, including three traditional unsupervised approaches (DeepWalk \cite{Perozzi_2014}, VGAE \cite{Kipf2016VariationalGA}, Node2Vec \cite{10.1145/2939672.2939754}) and three self-supervised approaches (DGI \cite{veličković2018deep}, MVGRL \cite{pmlr-v119-hassani20a}, GraphMAE \cite{hou2022graphmae}), as well as four multiplex methods, including one unsupervised (MNE \cite{ijcai2018p428}) and three self-supervised approaches (HDMI \cite{Jing_2021}, MCGC \cite{NEURIPS2021_10c66082}, DMG \cite{pmlr-v202-mo23a}). Finally, we include graph2vec \cite{Narayanan2017graph2vecLD}, which directly learns graph-level embeddings.

\begin{table*}[!!!!!!!!!!!!!t]
\caption{Comparing performance (Macro- and Micro-F1) of embeddings across selected and layer-specific labels on Syn1.}
\label{tab1}
\centering
\renewcommand{\arraystretch}{1.5}
\resizebox{2.07\columnwidth}{!}{%
\begin{tabular}{l cc cc cc cc}
\toprule
\textbf{Embeddings} & \multicolumn{2}{c}{\textbf{Final labels}} & \multicolumn{2}{c}{\textbf{Labels of layer 1}} & \multicolumn{2}{c}{\textbf{Labels of layer 2}} & \multicolumn{2}{c}{\textbf{Labels of layer 3}} \\
 & \textbf{Macro-F1} & \textbf{Micro-F1} & \textbf{Macro-F1} & \textbf{Micro-F1} & \textbf{Macro-F1} & \textbf{Micro-F1} & \textbf{Macro-F1} & \textbf{Micro-F1} \\
\midrule
Combined embeddings & \textbf{0.8178 $\pm$ 0.0516} & \textbf{0.8200 $\pm$ 0.0510} & \textbf{0.9611 $\pm$ 0.0476} & \textbf{0.9600 $\pm$ 0.0490} & \textbf{0.9799 $\pm$ 0.0246} & \textbf{0.9800 $\pm$ 0.0245} & \textbf{0.9770 $\pm$ 0.0282} & \textbf{0.9800 $\pm$ 0.0245} \\
Private embeddings of the 1st layer & \textbf{0.8158 $\pm$ 0.0545} & \textbf{0.8202 $\pm$ 0.0511} & \textbf{0.9904 $\pm$ 0.0191} & \textbf{0.9900 $\pm$ 0.0200} & 0.4714 $\pm$ 0.0917 & 0.4900 $\pm$ 0.0663 & 0.2463 $\pm$ 0.0533 & 0.2600 $\pm$ 0.0374 \\
Private embeddings of the 2nd layer & 0.1946 $\pm$ 0.0684 & 0.2200 $\pm$ 0.0678 & 0.2502 $\pm$ 0.1208 & 0.2600 $\pm$ 0.1158 & \textbf{0.9476 $\pm$ 0.0032} & \textbf{0.9500 $\pm$ 0.0000} & 0.2774 $\pm$ 0.0659 & 0.3500 $\pm$ 0.0548 \\
Private embeddings of the 3rd layer & 0.3248 $\pm$ 0.0483 & 0.3500 $\pm$ 0.0316 & 0.3063 $\pm$ 0.0274 & 0.3200 $\pm$ 0.0245 & 0.3119 $\pm$ 0.0896 & 0.3500 $\pm$ 0.0949 & \textbf{0.9791 $\pm$ 0.0260} & \textbf{0.9800 $\pm$ 0.0245} \\
\bottomrule
\end{tabular}
}
\end{table*}

\emph{Evaluation on downstream tasks.} We evaluate performance on node classification, node clustering, and graph classification tasks. 
For node clustering, we apply K-means to the learned node embeddings, and clustering quality is assessed using the Adjusted Rand Index (ARI) \cite{Hubert1985ComparingP} and Normalized Mutual Information (NMI) \cite{10.1162/153244303321897735}. For classification tasks, node- and graph-level embeddings are evaluated using a supervised classifier (a single-layer perceptron). To integrate common and private embeddings, we employ a Multi-Head Attention (MHA) combiner with four heads, whose outputs are averaged to produce the final representation. For fair comparison, single-view baselines are trained independently on each graph layer, and their embeddings are combined using the same combiner module. If a baseline already includes its own cross-layer fusion mechanism, we use the original implementation. To ensure robust evaluation, we adopt nested stratified K-fold cross-validation for both performance estimation and hyperparameter tuning: the outer loop provides an unbiased estimate of generalization performance, while the inner loop performs model selection without test leakage. Classification performance is measured using Macro-F1 and Micro-F1 scores. For feature construction, when only one shared feature matrix is available, we generate layer-specific features by randomly dropping entries from the shared feature matrix. When node features are unavailable, we use one-hot encodings of node identities.

\subsection{Synthetic dataset}
In graph data, shared and layer-specific information may arise from different components,
such as graph connectivity, node features, temporal dynamics, or their combinations. Accordingly, common embeddings should capture factors that are consistent across layers, whereas private embeddings should encode layer-dependent variations. Since these signals may be subtle or partially latent, the model must identify and disentangle both explicit and hidden patterns across and within layers.
We construct four synthetic multiplex datasets, denoted Syn1–Syn4, each isolating a different source of shared and layer-specific variation: (i) distinct graph structures with shared node features (Syn1), (ii) shared features and global structure with layer-specific perturbations (Syn2),
(iii) shared topology with layer-dependent graph signals, including variations in frequency-domain (Syn3) and nonlinear dynamical processes (Syn4). This suite provides controlled tests of disentanglement across structure, attributes, spectral content, and dynamics, spanning both node- and graph-level tasks.

\subsubsection{Distinct graph structures with shared node features}
In this setting, which we denote as \textbf{Syn1}, we design a synthetic multiplex graph where layers differ in structural organization while sharing common node features.
Each layer is constructed using a Stochastic Block Model (SBM), in which nodes are assigned to random communities and edges are sampled with higher probability for intra-community pairs than for inter-community pairs ($p_{\text{intra}}=0.7$, $p_{\text{inter}}=0.1$).
We generate $N = 3$ layers, each with its own community structure, containing three distinct communities, with a total of $M = 100$ nodes. Node features are initialized as identity matrices, and final node labels are obtained by probabilistically sampling from the label assignments of each layer, allowing us to model varying semantic interpretations across layers. Final labels are sampled with probabilities [0.8, 0.1, 0.1], making the resulting labeling predominantly aligned
with the first layer.

Table~\ref{tab1} evaluates whether our model learns private, layer-specific components in an interpretable way. 
As expected, private embeddings perform best on their respective layer-specific labels. 
In addition, the first-layer private embeddings achieve the strongest performance on the final labels, consistent with the fact that most labels originate from the first layer’s communities. This confirms that these embeddings effectively capture layer-specific properties while remaining separated from the shared embeddings.
Table \ref{tab2}
 reports results on final-label prediction, where CaDeM consistently outperforms all baselines. These improvements highlight the benefit of disentangling shared and layer-specific factors via the proposed objectives, producing more discriminative embeddings for prediction.

\begin{table}[!t]
\caption{Comparison of Macro- and Micro-F1 scores for the final label prediction across different baselines on Syn1.}
\label{tab2}
\centering
\footnotesize
\setlength{\tabcolsep}{12pt}
\renewcommand{\arraystretch}{1.1}
\resizebox{\columnwidth}{!}{%
\begin{tabular}{lcc}
\toprule
\textbf{Models} & \textbf{Macro F1} & \textbf{Micro F1} \\
\midrule
DeepWalk  & $0.7709 \pm 0.1075$ & $0.7800 \pm 0.0927$ \\
Node2Vec  & $0.3033 \pm 0.0737$ & $0.3300 \pm 0.0510$ \\
VGAE      & $0.7962 \pm 0.0294$ & $0.8000 \pm 0.0316$ \\
DGI       & $0.6782 \pm 0.1954$ & $0.6800 \pm 0.1913$ \\
MVGRL     & $0.7779 \pm 0.1098$ & $0.7800 \pm 0.1122$ \\
GraphMAE  & $0.7933 \pm 0.0600$ & $0.8000 \pm 0.0548$ \\
MNE       & $0.6899 \pm 0.0961$ & $0.7000 \pm 0.0894$ \\
HDMI      & $0.7842 \pm 0.0357$ & $0.7900 \pm 0.0374$ \\
MCGC      & $0.5906 \pm 0.0376$ & $0.6000 \pm 0.0447$ \\
DMG       & $0.7804 \pm 0.0303$ & $0.7800 \pm 0.0245$ \\
\midrule
\textbf{CaDeM} & \textbf{0.8178 $\pm$ 0.0516} & \textbf{0.8200 $\pm$ 0.0510} \\
\bottomrule
\end{tabular}
}
\end{table}

\subsubsection{Shared features and global structure with layer-specific perturbations}
Unlike the previous synthetic setting, where layers differed in graph structure but shared node features, \textbf{Syn2} considers a multiplex graph where all layers share a global community structure, while exhibiting layer-specific deviations. This setup reflects real-world multiplex systems that maintain a consistent overall organization but undergo local, layer-dependent reorganizations.
We first assign each node to a shared community label $c_i \in \{1, \dots, K\}$, defining the global structure across layers. The shared connectivity pattern is specified through a probability matrix $\mathcal{P}$, whose entries determine connection probabilities between communities.
In our experiments, $\mathcal{P}$ assigns high intra-community connection probabilities ($p_\text{intra} = 0.3$) and uniformly low inter-community probabilities ($p_\text{inter} = 0.01$).
To introduce layer-specific variations, we construct for each layer $\ell$ a private label set $s^{(\ell)}$  by randomly reassigning  
$50\%$ of the node labels.  
Each layer also defines its private connectivity matrix $Q^{(\ell)}$, which we set equal to $\mathcal{P}$ for simplicity. Edges are then sampled independently according to a convex combination of the shared and private structures, i.e., $p[(i,j) \in E^{(\ell)}] = \gamma \mathcal{P}_{c_i,c_j} + (1 - \gamma) Q^{(\ell)}_{s^{(\ell)}_i,s^{(\ell)}_j},
$
where $\gamma \in [0,1]$ controls the relative contribution of the shared and private component. We consider the one-hot encoding of the nodes as features, and generate a multiplex graph of $M = 1000$ nodes, $N = 3$ layers, and 3 communities per layer. 

For the experimental evaluation, we apply K-means to both the common and private embeddings and assess how effectively these embeddings recover the ground truth shared 
and layer-specific community structure, respectively. To this end, we vary the parameter $\gamma$ from $0$ to $1$, thereby controlling the trade-off between shared and private structural influence during data generation.  
Fig. \ref{figsyn2} presents the resulting trend of 
ARI as a function of $\gamma$. Similar behavior is observed for NMI.
\begin{figure}[!!!t]
  \centering
  
    \includegraphics[width=0.97\linewidth,valign=c]{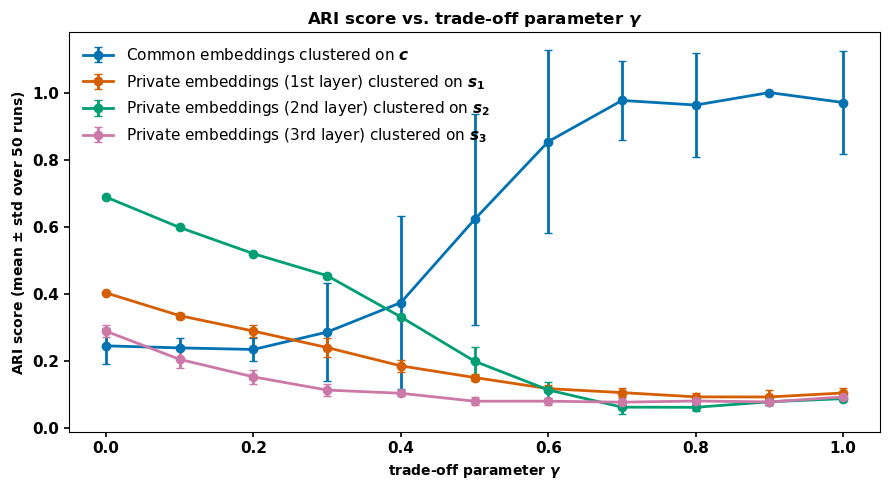}
  
  \caption{Clustering performance as a function of the trade-off parameter $\gamma$ on Syn2. This plot reports the ARI score obtained by applying K-means clustering to the learned embeddings.}
  \label{figsyn2}
\end{figure}
As illustrated in the plots, increasing the value of $\gamma$ improves the ability of the common embeddings to recover the shared community labels, as the shared structural component increasingly dominates the network organization.  
Conversely, a larger $\gamma$ reduces the strength of layer-specific variations, making the layer-specific labels harder for the private embeddings to recover. These results demonstrate that CaDeM effectively disentangles common and private embeddings according to their specific roles. As $\gamma$ varies, the two types of embeddings exhibit distinct performance trends, indicating minimal interference between them.

\begin{figure*}[!t]
  \centering

  \begin{minipage}[t]{0.33\textwidth}
    \centering
    \textbf{(a)}
    \includegraphics[width=0.8\linewidth]{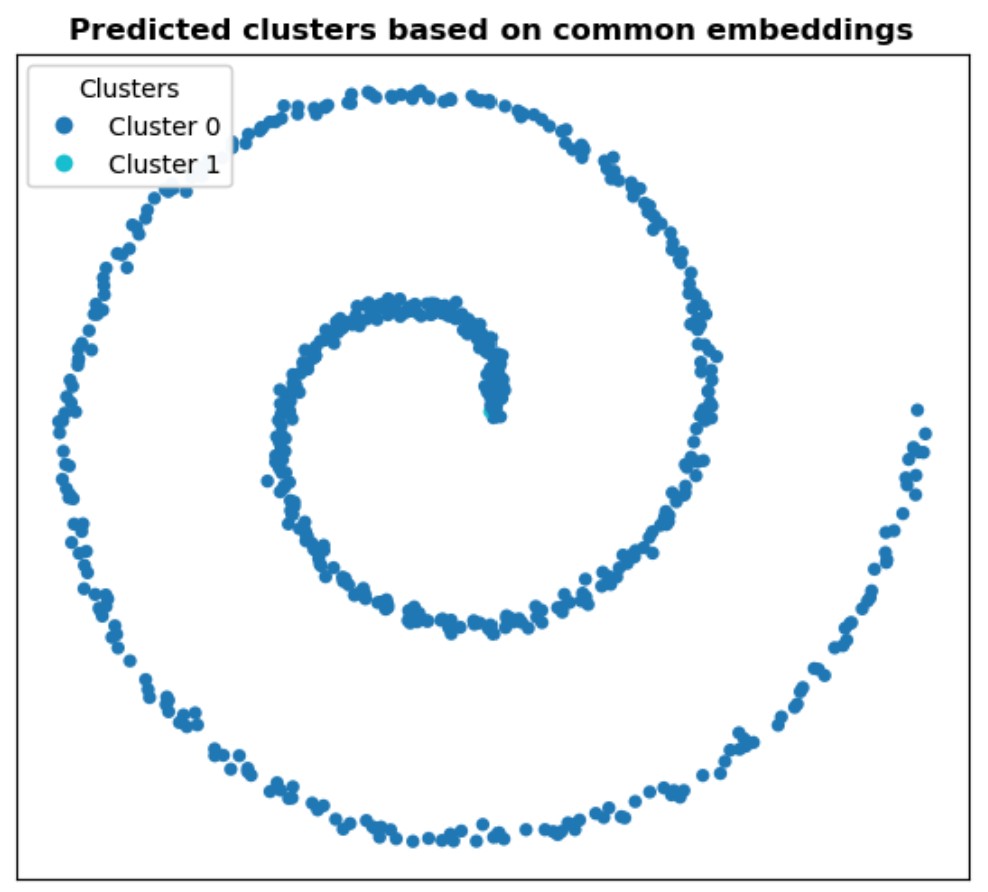}
    \label{figsyn3second3}
  \end{minipage}\hfill
  \begin{minipage}[t]{0.33\textwidth}
    \centering
    \textbf{(b)}
    \includegraphics[width=0.8\linewidth]{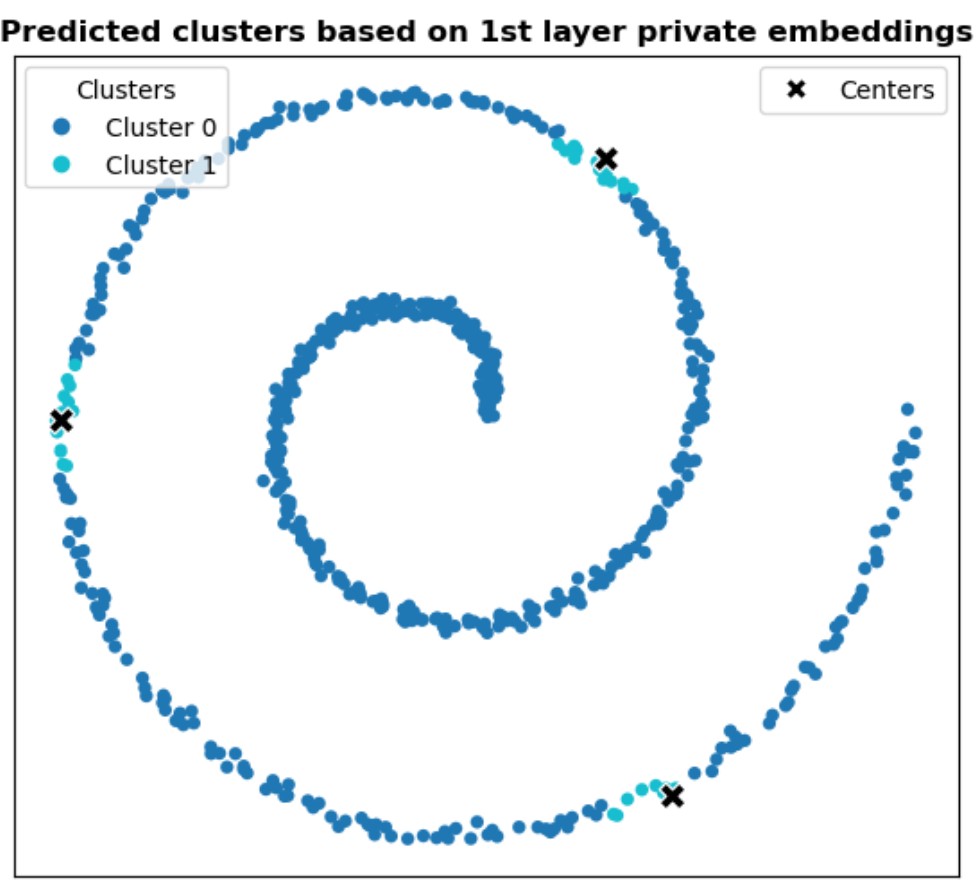}
    \label{figsyn3second4}
  \end{minipage}\hfill
  \begin{minipage}[t]{0.33\textwidth}
    \centering
    \textbf{(c)}
    \includegraphics[width=0.8\linewidth]{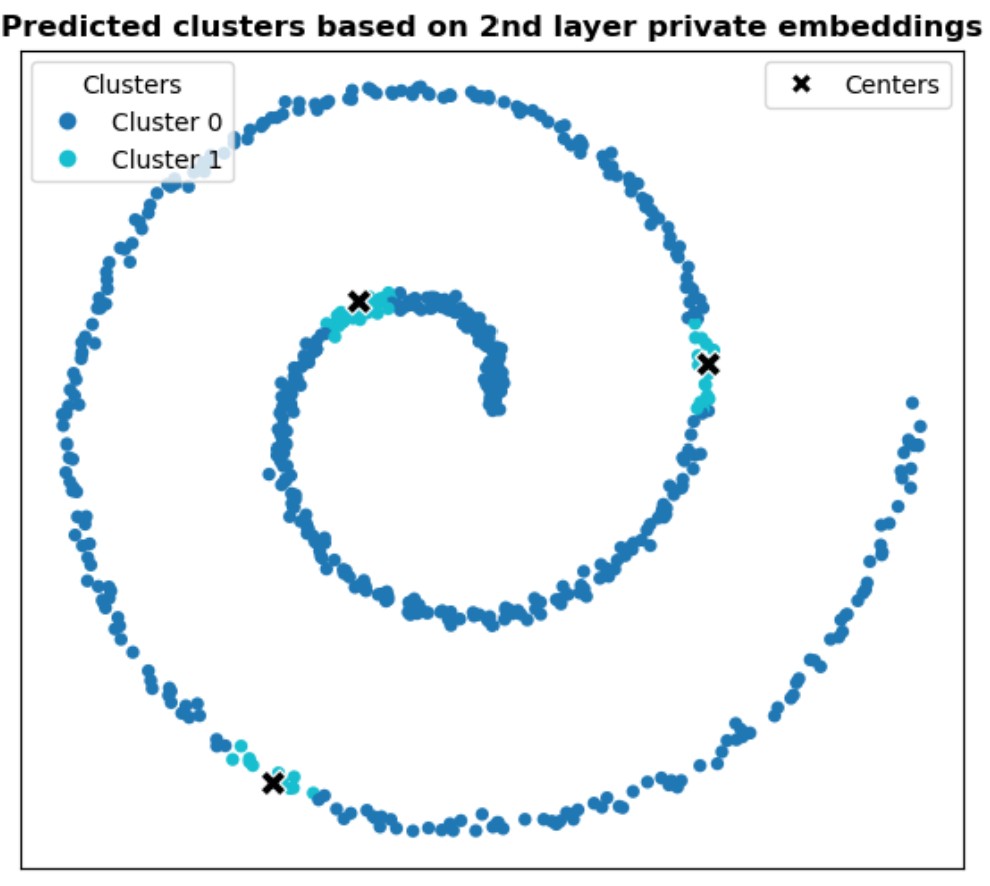}
    \label{figsyn3second5}
  \end{minipage}

  \caption{Visualization of predicted cluster assignments for Syn3. Panel (a) shows the clusters obtained from the common embeddings. Panels (b) and (c) show the clusters obtained from the private embeddings
of layer 1 and layer 2, respectively, with the corresponding ground-truth center nodes overlaid.}
  \label{figsyn3second}
\end{figure*}

\begin{table}[!t]
\caption{Clustering performance (ARI and NMI) of embeddings evaluated on the true clusters of each layer in Syn3.}
\label{tabsyn3}
\setlength{\tabcolsep}{15pt}
\renewcommand{\arraystretch}{1.5}
\resizebox{\columnwidth}{!}{%
\begin{tabular}{lcccc}
\toprule
\multirow{2}{*}{\textbf{Embeddings}} & 
\multicolumn{2}{c}{\textbf{Layer 1 clusters}} & 
\multicolumn{2}{c}{\textbf{Layer 2 clusters}} \\
\cmidrule(lr){2-3} \cmidrule(lr){4-5}
 & \textbf{ARI} & \textbf{NMI} & \textbf{ARI} & \textbf{NMI} \\
\midrule
Common embedding  
& 0.0099 $\pm$ 0.0590 
& 0.0077 $\pm$ 0.0227  
& 0.0362 $\pm$ 0.1345 
& 0.0261 $\pm$ 0.0910 \\

Private embedding of the 1st layer 
& \textbf{0.4935 $\pm$ 0.0286} 
& \textbf{0.4066 $\pm$ 0.0234} 
& -0.0734 $\pm$ 0.0267 
& 0.0348 $\pm$ 0.0174 \\

Private embedding of the 2nd layer 
& -0.0911 $\pm$ 0.0911 
& 0.0450 $\pm$ 0..0450 
& \textbf{0.6441 $\pm$ 0.0644 } 
& \textbf{0.5405 $\pm$ 0.0135} \\
\bottomrule
\end{tabular}
}
\end{table}

\subsubsection{Layer-Specific features with spectral variations}
Next, we design a synthetic dataset, denoted as \textbf{Syn3}, in which layers share the same underlying structure but exhibit distinct graph signal frequency content. 
This setup emulates scenarios where multiplex systems share a common low-frequency signal while differing in localized high-frequency components. We construct a synthetic graph on a 2-dimensional spiral manifold and build a shared $K$-nearest-neighbor graph ($K=15$) to preserve spatial locality. From this graph, we compute the combinatorial Laplacian $L$.

Node features are composed of a shared low-frequency component $x_{\text{common}}$, capturing smooth patterns, and a layer-specific component $x_{P_{\ell}}$ which encodes localized, high-frequency variations unique to each layer. $x_{\text{common}}$ is generated by applying a low-pass filter $h(\lambda) = e^{-\lambda}$ to a Gaussian random vector, producing a smooth signal over the graph, where $\lambda$ denotes a Laplacian eigenvalue.
To generate layer-specific signals, we define a sharp band-pass filter $g(\lambda) = \lambda e^{-\lambda^2}$. For each layer $\ell$, a set of center nodes is selected along the spiral manifold, and for each center $r_\ell$, a local multi-hop neighborhood $\mathcal{N}(r_\ell)$ is extracted using breadth-first search (BFS). To ensure smooth spatial decay, each node $q \in \mathcal{N}(r_\ell)$ is assigned a cosine-taper weight $w_q$. Localized wavelet atoms are then constructed as $\psi_{t,q} = g(tL)\delta_q$,
where $t > 0$ controls the scale of localization and $\delta_q$ denotes the Kronecker delta centered at node $q$.
The layer-specific signal for layer $\ell$ is defined as a superposition of localized wavelets: $x_{P_\ell}
=
\sum_{r_\ell}
\sum_{q \in \mathcal{N}(r_\ell)}
w_q \psi_{t, {q}}
$.
Node labels in each layer are assigned based on hop distance from the set of layer-specific center nodes. Nodes within a fixed radius (3 hops) are grouped into the same class, while all remaining nodes are treated as background.
The final layer signal is formed as $\mathbf{x}_\ell = 0.3\, \mathbf{x}_{\mathrm{common}} + \mathbf{x}_{P_\ell}$,
yielding a dataset with $M = 600$ nodes, $N = 2$ layers, and $3$ center per layer. 

For evaluation, we apply the K-means clustering algorithm with $K=2$ to both the common and private embeddings to analyze how well each representation detects layer-specific patterns. Fig.~\ref{figsyn3second} displays the predicted clusters obtained from each embedding. The quantitative results are summarized in Table~\ref{tabsyn3}.
As shown in Fig. \ref{figsyn3second} and Table~\ref{tabsyn3}, the common embeddings fail to reveal meaningful clustering since they intentionally discard layer-specific features.
Although they may encode shared structural properties, K-means cannot leverage structural information to detect feature patterns in this case. 
In contrast, the private embeddings successfully recover the layer-specific spectral patterns.
These results indicate that the common and private embeddings encode complementary information, each preserving the signals required for its purpose, even when the two layers share the same graph structure and differ only in their graph signals. 

\subsubsection{Layer-Specific Features with Dynamical Variations}
We evaluate our method on a synthetic dataset (\textbf{Syn4}) for graph-level tasks, where layers share the same topology but differ in the nonlinear dynamical systems governing the node-state evolution, enabling controlled analysis of layer-specific dynamics.

We generate a collection of 30 connected Erdős–Rényi (ER) graphs \cite{Bollobas2001} with $M = 100$ nodes and edge probability $p=0.08$, each serving as the shared topology of a 2-layer multiplex graph ($N=2$).
To avoid structural bias in classification, the same set of base graphs is reused to generate samples across all dynamical classes. We consider three major classes of dynamics: degree-driven, homogeneous, and degree-avert \cite{harush2017dynamic}. 
For each class, two nonlinear dynamical systems are simulated: population \cite{novozhilov2006biological} and regulatory dynamics \cite{karlebach2008modelling} for the degree-driven class; epidemic \cite{moreno2002epidemic, hufnagel2004forecast, dodds2005generalized} and biochemical dynamics \cite{voit2000computational, PhysRevLett.106.150602} for the homogeneous class; and mutualistic \cite{Holling1959SomeCO} and regulatory dynamics \cite{karlebach2008modelling} for the degree-avert class.
Each dynamical process produces a unique temporal evolution of node states. 
For each multiplex sample, we simulate the same dynamical class on both layers but use two different dynamics from that class and random initial conditions, producing a total of $90$ multiplex graphs.
All dynamical processes
are simulated using fixed-step numerical integration. 

For graph-level prediction, node embeddings are aggregated via mean pooling to obtain common and private graph embeddings, which are used to predict the dynamical class.
To ensure fair evaluation, we group all multiplex graphs derived from the same base topology within the same fold in stratified nested cross-validation. Table \ref{tab:data4_merged} reports the 
results for classifying multiplex graphs into the three dynamical classes. In addition, we evaluate layer-level discrimination by separating the two layers of each multiplex sample, yielding $180$ single-layer graphs, and classifying them into the six dynamical subclasses using pooled private embeddings, as reported in the last row of Table~\ref{tab:data4_merged}. 
Results in Table \ref{tab:data4_merged} show that common embeddings perform near randomly when predicting dynamical classes, as expected since they are designed to exclude layer-specific information. In contrast, private embeddings achieve near-perfect performance, demonstrating their ability to capture fine-grained, layer-specific dynamical signatures, including distinctions among subclasses within the same class, with minimal interference.
Combined embeddings also perform well, although their scores are slightly lower than those of the private embeddings alone. This decrease is expected, as the combined representation incorporates some contributions from the common embeddings, which contain information irrelevant to the dynamical class. Overall, these results confirm that the joint objective enforces a clear separation between common and private embeddings, with each capturing information relevant to its role. 

\begin{table}[t]
    \centering
    \caption{Classification results on Syn4 for different tasks and embedding types.}
    \label{tab:data4_merged}
    \small
    \resizebox{\columnwidth}{!}{%
    \begin{tabular}{p{4.2cm} l cc}
        \toprule
        \textbf{Task} & \textbf{Embeddings} & \textbf{Macro-F1 score} & \textbf{Micro-F1 score} \\
        \midrule
        Multiplex graph classification\newline (3 dynamical classes)
        & Combined graph embeddings                & %$0.8401 \pm 0.0365$ & $0.8427 \pm 0.0367$ \\
        $0.9329 \pm 0.0548$ & $0.9333  \pm 0.0544$ \\
        
        & Common graph embeddings                  & $0.3844 \pm 0.0790$ & $0.3889 \pm 0.0786$ \\
        & Private graph embeddings (1st layer)     & $0.9892 \pm 0.0054$ & $0.9893 \pm 0.0053$ \\
        & Private graph embeddings (2nd layer)     & $0.9888 \pm 0.0224$ & $0.9889 \pm 0.0222$ \\
        \midrule
        Single-layer graph classification\newline (6 dynamical subclasses)
        & Private graph embeddings           & $0.9944 \pm 0.0112$ & $0.9944 \pm 0.0111$ \\
        \bottomrule
    \end{tabular}%
    }
\end{table}

\subsection{Real-world datasets}

 We next evaluate CaDeM on real-world multiplex benchmarks, where shared and layer-specific factors are unknown and often entangled.

\begin{figure*}[!t]
  \centering

  % --- First row ---
  \begin{minipage}[t]{0.5\textwidth}
    \centering
    \includegraphics[width=\linewidth]{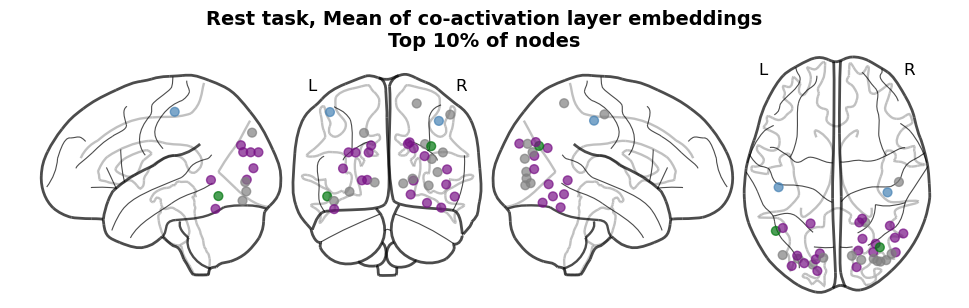}
    \label{hcp-rest-co}
  \end{minipage}\hfill
  \begin{minipage}[t]{0.5\textwidth}
    \centering
    \includegraphics[width=\linewidth]{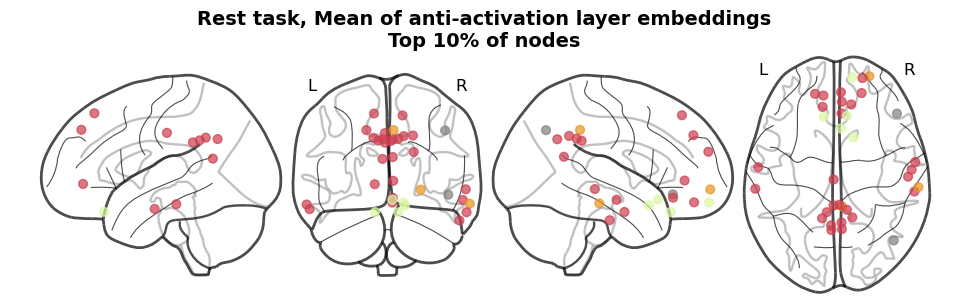}
    \label{hcp-rest-anti}
  \end{minipage}

  \vspace{0.01cm}

  % --- Second row ---
  \begin{minipage}[t]{0.5\textwidth}
    \centering
    \includegraphics[width=\linewidth]{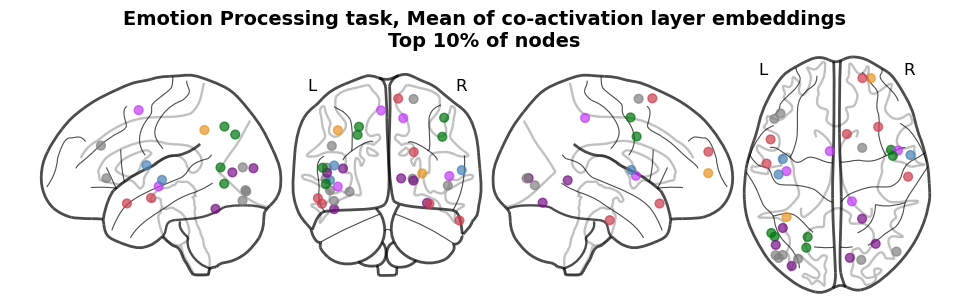}
    \label{hcp-emo-co}
  \end{minipage}\hfill
  \begin{minipage}[t]{0.5\textwidth}
    \centering
    \includegraphics[width=\linewidth]{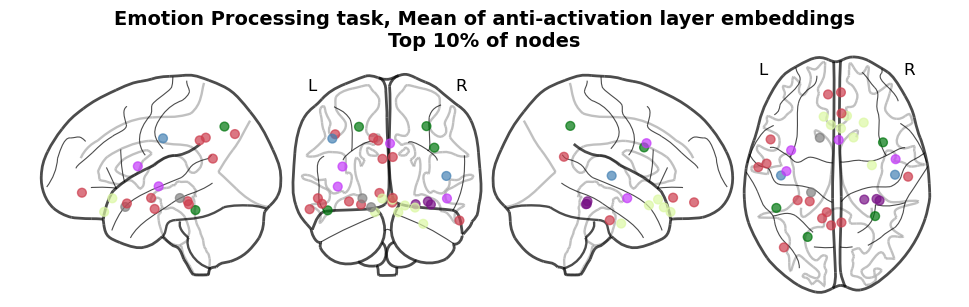}
    \label{hcp-emo-anti}
  \end{minipage}

  \vspace{0.01cm}

  % --- Third row ---
  \begin{minipage}[t]{0.5\textwidth}
    \centering
    \includegraphics[width=\linewidth]{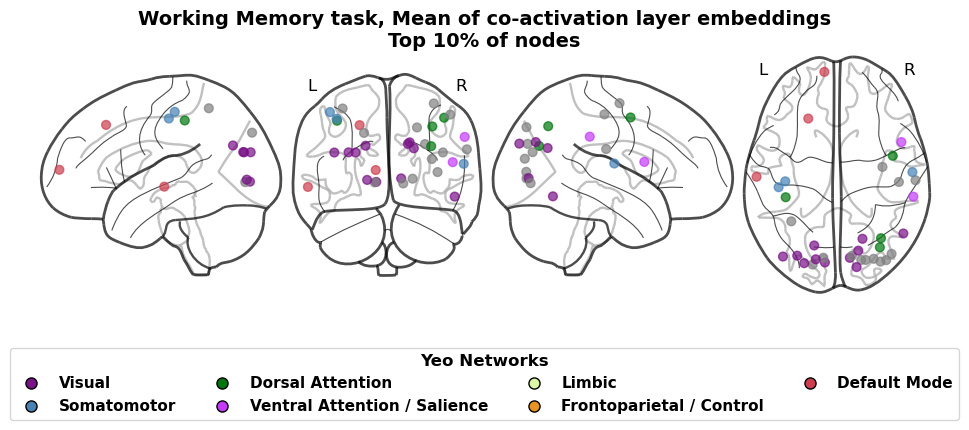}
    \label{hcp-wm-co}
  \end{minipage}\hfill
  \begin{minipage}[t]{0.5\textwidth}
    \centering
    \includegraphics[width=\linewidth]{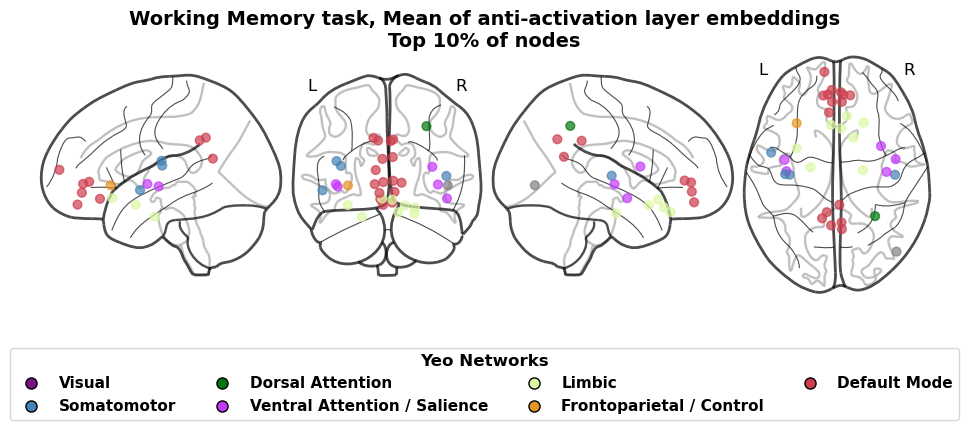}
    \label{hcp-wm-anti}
  \end{minipage}

  \caption{For each task, we display the top $10\%$ of nodes based on the mean co-activation (left) and anti-activation (right) layer embeddings across subjects. Nodes are projected onto the cortical surface and colored according to the Yeo network atlas, revealing distinct task-dependent patterns across functional brain networks.}
  \label{hcp-interp}
\end{figure*}

\subsubsection{Graph classification}
We evaluate our framework on fMRI data from the Human Connectome Project (HCP) S1200 release \cite{VanEssen2013HCP}.
We use the BOLD time series derived from the Multi-Modal Parcellation (MMP) with 360 brain regions (nodes) \cite{glasser2016multi} extracted from the preprocessed HCP data. 
The time series are further processed using commonly adopted preprocessing for functional connectivity estimation, including linear detrending, band-pass filtering (0.008–0.08 Hz), confound regression, and z-score standardization of the extracted BOLD signals. The resulting signals capture low-frequency neural fluctuations and are used to construct the functional connectivity (FC) matrix by computing the Pearson correlation between all regions. In addition, we compute parcel-level skewness and kurtosis as node features characterizing the structure of regional neural activity.
The FC matrix for each subject and task contains positive and negative correlations between brain regions, representing co-activation, in which regions fluctuate together, and anti-activation, in which regions exhibit opposing activity patterns. We therefore split the matrix into a co-activation component using positive correlations and an anti-activation component using the absolute values of negative correlations.
Together, these two components define a 2-layer multiplex graph ($N=2$) for each subject and task.

We focus on three tasks, Rest, Emotion Processing, and Working Memory, covering an unconstrained baseline state and two cognitively demanding conditions.
We select $150$ subjects and use the same set across all tasks, yielding 450 two-layer multiplex graphs, with skewness and kurtosis as node features. 
Graph-level embeddings are obtained using sum pooling on the CaDeM embeddings.
Each multiplex graph is classified into one of the three tasks, and performance is evaluated using nested cross-validation with subject-level splits, ensuring that all graphs from a given subject appear exclusively in either the training or the test set.

\begin{table}[t]
\centering
\caption{Comparison of Macro- and Micro-F1 scores across baselines on the HCP dataset for classifying multiplex graphs into Rest, Emotion, and Working Memory tasks.}
\renewcommand{\arraystretch}{1}
\resizebox{0.9\columnwidth}{!}{
\begin{tabular}{lcc}
\toprule
\textbf{Method} & \textbf{Macro F1} & \textbf{Micro F1} \\
\midrule
MVGRL     & 0.3338 $\pm$ 0.0529 & 0.3365 $\pm$ 0.0520 \\
DGI       & 0.6009 $\pm$ 0.0208 & 0.6027 $\pm$ 0.0240 \\
HDMI      & 0.4038 $\pm$ 0.0272 & 0.4188 $\pm$ 0.0256 \\
DMG       & 0.3983 $\pm$ 0.0432 & 0.4032 $\pm$ 0.0443 \\
Graph2Vec & 0.6033 $\pm$ 0.0264 & 0.6192 $\pm$ 0.0325 \\
\bottomrule
\textbf{CaDeM} & \textbf{0.6628 $\pm$ 0.0277} & \textbf{0.6659 $\pm$ 0.0287} \\
\end{tabular}
}
\label{tabHCP}
\end{table}

As shown in Table \ref{tabHCP}, CaDeM outperforms all baselines, achieving the highest F1 scores across the three tasks. These results demonstrate that the model learns meaningful representations that capture task-specific patterns. The performance gains can be attributed to our causality-driven design, which reduces noise in the learned embeddings and minimizes interference between common and private representations.
Overall, the results indicate that the approach generalizes well to real-world graph-level tasks such as brain task classification.

\textbf{Interpretability.} To interpret the learned graph-level representations, we average task-specific private embeddings of the co-activation and anti-activation layers across subjects and visualize the top $10\%$ of nodes ranked by absolute embedding magnitude on a cortical surface using the Yeo network atlas \cite{Yeo2011}. The visualizations are presented in Fig. \ref{hcp-interp}. Across tasks, the highlighted regions exhibit patterns that are consistent with known functional brain organization in both co-activation and anti-activation layers. During the Rest state, the co-activation layer prominently highlights regions within the visual and sensorimotor networks, reflecting coherent intrinsic activity within primary sensory systems. In the anti-activation layer, we observe strong involvement of the default mode network (DMN) and task-positive attention networks, including the frontoparietal network (FPN), consistent with the well-established anticorrelation between DMN and attention/control networks during rest \cite{di2014modulatory,uddin2009functional}. 
For the emotion task, co-activation patterns emphasize regions associated with salience detection and affective processing, including the salience network and frontoparietal network \cite{schimmelpfennig2023role}, aligning with the engagement of emotional salience and cognitive control mechanisms. In the anti-activation layer, we observe prominent involvement of DMN, consistent with the dynamic switching role of the salience network between internally-focused DMN activity and externally-directed attention processes \cite{schimmelpfennig2023role} during emotional engagement.
In the working memory task, co-activation patterns primarily involve DAN and VAN, consistent with the need to maintain goal-directed attention and hold task-relevant information during working memory \cite{uddin2009functional}. The anti-activation layer highlights regions within the DMN and limbic networks, indicating suppression of internally-focused processing during externally-oriented cognitive demands \cite{anticevic2012role}. These task-specific activation patterns demonstrate that the proposed framework learns embeddings that are not only discriminative for classification, but also neurobiologically meaningful, capturing established functional interactions across co-activation and anti-activation networks.

\begin{table*}[!t]
\caption{Performance comparison of baseline methods for node classification on four multiplex datasets.}
\label{tab3}
\centering
\renewcommand{\arraystretch}{1.4}
\resizebox{2.1\columnwidth}{!}{%
\begin{tabular}{l cc cc cc cc}
\toprule
\textbf{Method} & \multicolumn{2}{c}{\textbf{DBLP}} & \multicolumn{2}{c}{\textbf{IMDB}} & \multicolumn{2}{c}{\textbf{ACM}} & \multicolumn{2}{c}{\textbf{Freebase}} \\
 & \textbf{Macro-F1} & \textbf{Micro-F1} & \textbf{Macro-F1} & \textbf{Micro-F1} & \textbf{Macro-F1} & \textbf{Micro-F1} & \textbf{Macro-F1} & \textbf{Micro-F1} \\
\midrule
DeepWalk   & 0.5488 $\pm$ 0.0130 & 0.5653 $\pm$ 0.0134 & 0.4651 $\pm$ 0.0283 & 0.4794 $\pm$ 0.0276 & 0.8505 $\pm$ 0.0167 & 0.8499 $\pm$ 0.0157 & 0.5993 $\pm$ 0.0203 & 0.6638 $\pm$ 0.0195 \\
Node2Vec   & 0.5433 $\pm$ 0.0033 & 0.5681 $\pm$ 0.0037 & 0.5663 $\pm$ 0.0175 & 0.6000 $\pm$ 0.0123 & 0.7249 $\pm$ 0.0178 & 0.7226 $\pm$ 0.0177 & 0.4732 $\pm$ 0.0145 & 0.6254 $\pm$ 0.0156 \\
VGAE       & 0.8164 $\pm$ 0.0105 & 0.8080 $\pm$ 0.0101 & 0.6278 $\pm$ 0.0142 & 0.6301 $\pm$ 0.0129 & 0.9093 $\pm$ 0.0084 & 0.9084 $\pm$ 0.0085 & 0.5523 $\pm$ 0.0112 & 0.6375 $\pm$ 0.0052 \\
DGI        & 0.7283 $\pm$ 0.0130 & 0.7199 $\pm$ 0.0120 & 0.5213 $\pm$ 0.0658 & 0.5349 $\pm$ 0.0548 & 0.9107 $\pm$ 0.0128 & 0.9101 $\pm$ 0.0132 & 0.5426 $\pm$ 0.0116 & 0.6629 $\pm$ 0.0072 \\
MVGRL      & 0.7455 $\pm$ 0.0063 & 0.7310 $\pm$ 0.0073 & 0.6094 $\pm$ 0.0120 & 0.6101 $\pm$ 0.0110 & 0.9105 $\pm$ 0.0032 & 0.9098 $\pm$ 0.0032 & 0.5102 $\pm$ 0.0210 & 0.6438 $\pm$ 0.0167 \\
GraphMAE   & 0.8359 $\pm$ 0.0140 & 0.8269 $\pm$ 0.0133 & 0.6417 $\pm$ 0.0066 & 0.6417 $\pm$ 0.0052 & 0.9168 $\pm$ 0.0132 & 0.9160 $\pm$ 0.0129 & 0.5924 $\pm$ 0.0349 & 0.6652 $\pm$ 0.0337 \\
MNE        & 0.5661 $\pm$ 0.0113 & 0.5622 $\pm$ 0.0125 & 0.5522 $\pm$ 0.0155 & 0.5743 $\pm$ 0.0122 & 0.8492 $\pm$ 0.0053 & 0.8497 $\pm$ 0.0032 & 0.5411 $\pm$ 0.0201 & 0.6443 $\pm$ 0.0131 \\
HDMI       & 0.8350 $\pm$ 0.0125 & 0.8245 $\pm$ 0.0116 & 0.6484 $\pm$ 0.0158 & 0.6476 $\pm$ 0.0156 & 0.9026 $\pm$ 0.0086 & 0.9015 $\pm$ 0.0089 & 0.5899 $\pm$ 0.0161 & 0.6595 $\pm$ 0.0149 \\
MCGC       & 0.5308 $\pm$ 0.0057 & 0.6412 $\pm$ 0.0069 & 0.5637 $\pm$ 0.0302 & 0.5811 $\pm$ 0.0293 & 0.8536 $\pm$ 0.0180 & 0.8522 $\pm$ 0.0180 & 0.5961 $\pm$ 0.0113 & 0.6642 $\pm$ 0.0140 \\
DMG        & 0.8454 $\pm$ 0.0136 & 0.8372 $\pm$ 0.0131 & 0.6565 $\pm$ 0.0165 & 0.6566 $\pm$ 0.0156 & 0.9102 $\pm$ 0.0077 & 0.9094 $\pm$ 0.0076 & 0.6084 $\pm$ 0.0102 & 0.6795 $\pm$ 0.0100 \\
\midrule
\textbf{CaDeM} & \textbf{0.8497 $\pm$ 0.0111} & \textbf{0.8414 $\pm$ 0.0108} & \textbf{0.6667 $\pm$ 0.0158} & \textbf{0.6682 $\pm$ 0.0155} & \textbf{0.9287 $\pm$ 0.0084} & \textbf{0.9283 $\pm$ 0.0085} & \textbf{0.6224 $\pm$ 0.0121} & \textbf{0.6887 $\pm$ 0.0137} \\
\bottomrule
\end{tabular}%
}
\end{table*}

\begin{table*}[t]
\caption{Ablation study on different loss functions across one synthetic and three real-world datasets. Results are reported as Macro-F1 and Micro-F1 scores on the node classification task.}
\centering
\label{tabsensitivity}
\small
\setlength{\tabcolsep}{8pt}
\renewcommand{\arraystretch}{1.7}
\resizebox{2.1\columnwidth}{!}{%
\begin{tabular}{ccccccccccc}
\toprule
\multirow{2}{*}{$\mathcal{L}_\textit{matching}$} &
\multirow{2}{*}{$\mathcal{L}_\textit{self-supervised}$} &
\multirow{2}{*}{$\mathcal{L}_\textit{causal}$} &
\multicolumn{2}{c}{\textbf{Syn1}} &
\multicolumn{2}{c}{\textbf{IMDB}} &
\multicolumn{2}{c}{\textbf{ACM}} &
\multicolumn{2}{c}{\textbf{Freebase}} \\
\cmidrule(lr){4-5}\cmidrule(lr){6-7}\cmidrule(lr){8-9}\cmidrule(lr){10-11}
 & & &
\textbf{Macro} & \textbf{Micro} &
\textbf{Macro} & \textbf{Micro} &
\textbf{Macro} & \textbf{Micro} &
\textbf{Macro} & \textbf{Micro} \\
\midrule
\ding{51} & \ding{55} & \ding{55} &
0.7450 $\pm$ 0.0520 & 0.7500 $\pm$ 0.0447 &
0.6145 $\pm$ 0.0196 & 0.6187 $\pm$ 0.0180 &
0.8658 $\pm$ 0.0142 & 0.8651 $\pm$ 0.0142 &
0.5760 $\pm$ 0.0187 & 0.6607 $\pm$ 0.0136 \\
\ding{51} & \ding{51} & \ding{55} &
0.7673 $\pm$ 0.1500 & 0.7700 $\pm$ 0.1503 &
0.6599 $\pm$ 0.0244 & 0.6617 $\pm$ 0.0233 &
0.9152 $\pm$ 0.0106 & 0.9147 $\pm$ 0.0106 &
0.5807 $\pm$ 0.0188 & 0.6615 $\pm$ 0.0109 \\
\ding{51} & \ding{55} & \ding{51} &
0.7844 $\pm$ 0.0477 & 0.7800 $\pm$ 0.0510 &
0.6576 $\pm$ 0.0125 & 0.6580 $\pm$ 0.0109 &
0.9169 $\pm$ 0.0100 & 0.9164 $\pm$ 0.0101 &
0.6109 $\pm$ 0.0101 & 0.6793 $\pm$ 0.0035 \\
\ding{51} & \ding{51} & \ding{51} &
\textbf{0.8178 $\pm$ 0.0516} & \textbf{0.8200 $\pm$ 0.0510} &
\textbf{0.6667 $\pm$ 0.0158} & \textbf{0.6682 $\pm$ 0.0155} &
\textbf{0.9287 $\pm$ 0.0084} & \textbf{0.9283 $\pm$ 0.0085} &
\textbf{0.6224 $\pm$ 0.0121} & \textbf{0.6887 $\pm$ 0.0137} \\
\bottomrule
\end{tabular}
}
\end{table*}

\subsubsection{Node classification}
Table~\ref{tab3} reports the node classification performance on four real-world datasets described in Section \ref{eval_setup}. Our method clearly achieves the best performance on the node classification task across all cases. Compared to single-view graph approaches, it delivers higher F1 scores, highlighting the advantages of multiplex graph models in capturing cross-layer correlations and uncovering hidden structural information for node representations. Furthermore, when compared against existing multiplex baselines, our model again achieves the strongest results. This improvement can be attributed to our causality-inspired disentanglement strategy, which reduces noise in the learned representations and aligns common embeddings across layers to capture global shared information. This conclusion is further supported by the ablation study on the individual loss components, which we present in the following section.

\subsection{Sensitivity analysis}
\label{sec_sentivity}
To assess the contribution of each objective, we evaluate all model variants on the node classification task using our first synthetic dataset as well as the IMDB, ACM, and Freebase datasets, and report the results in Table~\ref{tabsensitivity}. Since the matching loss is essential for deriving aligned representations, it is included in all variants. 
Results show that adding either $\mathcal{L}_\textit{self-supervised}$ or $\mathcal{L}_\textit{causal}$ improves performance compared to using only the matching loss. This is consistent with their role in 
promoting disentanglement.
The best performance is achieved when all three losses are jointly applied, confirming their complementary roles in the framework.

\section{Conclusion}
\label{sec:conc}
In this work, we introduced a causality-inspired self-supervised framework for multiplex graph representation learning. By generalizing the notions of confounding effects to multiplex settings, we model private embeddings as causal signals and common embeddings as potential confounders for prediction in each layer. We then incorporate an objective inspired by the backdoor adjustment to separate their effects.
Our approach integrates three complementary losses to jointly enforce alignment and disentanglement of representations, and we theoretically justify this design. 
Experiments on both synthetic and real-world benchmarks demonstrate that our method consistently outperforms baselines, achieving more robust and interpretable embeddings. Looking ahead, this framework provides a principled basis for extending causal representation learning to more complex graph settings. Promising directions include integrating multiplex graphs with multi-modal data (e.g., text, images, or time series), as well as extending the approach to dynamic and temporal graphs where causal and confounding factors evolve over time. The framework is also applicable to decision-making and intervention-driven tasks, where separating shared from layer-specific components helps distinguish invariant mechanisms from context-dependent signals, reducing reliance on spurious global correlations in applications such as recommendation and network control.
The current formulation could be further improved in several respects. The theoretical analysis relies on idealized assumptions, such as sufficient model capacity and global optimization, 
which may not fully reflect practical training dynamics. In addition, the self-supervised loss uses the layer index as the prediction target. While effective, alternative targets may be beneficial in scenarios where layer differences are subtle or driven by more complex or continuous factors. Finally, the method assumes aligned node identities across layers and may require additional alignment mechanisms in partially overlapping multiplex graphs, which we leave for future work.

\bibliographystyle{IEEEtran}
\bibliography{references}

\clearpage

\section*{Supplementary Material for ``Causality-Driven Disentangled Representation Learning in Multiplex Graphs''}

\section{Theoretical Results and Proofs}
Throughout this section, we adopt an explicit probabilistic formulation. 

\textbf{Notation.} We denote random variables in \textbf{boldface}, and their realizations in standard (non-bold) notation. In particular, $\mathbf{L}$ denotes the discrete random variable representing the layer index, distributed uniformly over $\{1,\dots, N\}$: $\mathbf{L} \sim \mathrm{Unif}\{1,\dots, N\}$. We denote by $\mathbf{P}_{\mathbf{L}}$ and 
$\mathbf{C}_{\mathbf{L}}$, the private and common embedding random variables corresponding to the randomly drawn layer $\mathbf{L}$. For a realization $\mathbf{L}=\ell$, the corresponding realizations are denoted by $P_\ell$ and $C_\ell$.

To formalize the stratified pairing strategy of the causal loss, we introduce an additional layer-index random variable $\mathbf{L}'\sim \mathrm{Unif}\{1,\dots, N\}$, independent of $\mathbf{L}$, 
and define 
$\mathbf{C}_{\mathbf{L}'}$ as the common embedding random variable corresponding to the randomly drawn layer $\mathbf{L}'$. 

To define the matching loss, we introduce an auxiliary matrix $S \in \mathbb{R}^{M \times d}$ as a shared reference representation across layers, treated as an optimization variable.

The three loss terms are written in a probabilistic form as:

{\fontsize{8}{9}\selectfont
\begin{equation}
\label{eq2:matching}
\begin{aligned}
\mathcal{L}_{\text{matching}}
=\mathbb{E}_{\mathbf{C_L}}\!&\left[\| \mathbf{C_L} - S \|_F^2 \right], \hspace{3.5pt}
\text{s.t.}\hspace{3.5pt} 
S^{\top} S = I \in \mathbb{R}^{d \times d},\hspace{3.5pt} 
S^{\top} \mathbf{1} = \mathbf{0},
\end{aligned}
\end{equation}
}
\begin{equation}
\label{eq:self_theory}
\mathcal{L}_{\text{self-supervised}}
=
\mathbb{E}_{\mathbf{(L, P_L)}}\big[
-\log \phi(\mathbf{L} \mid \mathbf{P}_{\mathbf{L}})
\big],
\end{equation}
\begin{equation}
\label{eq:causal_theory}
\mathcal{L}_{\text{causal}}
=
\mathbb{E}_{(\mathbf{L, P_L, C_{L'}})}\big[
-\log \psi(\mathbf{L} \mid \mathbf{P}_{\mathbf{L}} \oplus \mathbf{C}_{\mathbf{L'}})
\big].
\end{equation}
In practice, the predictors $\phi$ and $\psi$ operate on pooled graph-level embeddings $h_{P_\ell}$ and $h_{C_\ell}$, obtained by applying a deterministic pooling operation to node-level embeddings $P_\ell$ and $C_\ell$. Since pooling is deterministic, the graph-level embeddings are measurable functions of the node-level embeddings and therefore introduce no additional randomness. Consequently, for theoretical analysis, it suffices to study
the random variables $\mathbf{P}_{\mathbf{L}}$ and 
$\mathbf{C}_{\mathbf{L}}$, abstracting away the  pooling operation.

\setcounter{theorem}{0}
\setcounter{proposition}{0}

\begin{proposition}
\label{theorem1}

\end{proposition}

\begin{proof}

In  $\mathcal{L}_{\text{self-supervised}}$  (Eq. \eqref{eq:self_theory}), we know by definition:
\begin{multline}
\mathbb{E}_{\mathbf{L}\mid \mathbf{P}_{\mathbf{L}}=P_\ell}
\!\left[-\log \phi(\mathbf{L}\mid P_\ell)\right]
=
\sum_{k=1}^{N}
\mathbb{P}\!\left(\mathbf{L}=k \mid \mathbf{P}_{\mathbf{L}}=P_\ell\right) \\
\cdot \big(-\log \phi(k\mid P_\ell)\big).
\end{multline}
We add and subtract the term $-\log p(k \mid P_\ell)$:
{\fontsize{8}{8}\selectfont
\begin{align}
\mathbb{E}_{\mathbf{L}\mid \mathbf{P}_{\mathbf{L}}=P_\ell}
\!\left[-\log \phi(\mathbf{L} \mid P_\ell)\right]
&= \sum_{k=1}^{N} 
p(\mathbf{L}=k \mid \mathbf{P}_{\mathbf{L}}=P_\ell)
\big(-\log p(k \mid P_\ell)\big) \notag \\
&\quad + \sum_{k=1}^{N} 
p(\mathbf{L}=k \mid \mathbf{P}_{\mathbf{L}}=P_\ell)
\log \frac{p(k \mid P_\ell)}{\phi(k \mid P_\ell)}.
\end{align}
}
Therefore, after taking the expectation over $\mathbf{P_L}$, we have
{\small
\begin{align}
\mathbb{E}_{\mathbf{P}_{\mathbf{L}}}
\Bigg[
  \mathbb{E}_{\mathbf{L}\mid \mathbf{P}_{\mathbf{L}}}
  \big[
    -\log \phi(\mathbf{L} \mid &\mathbf{P}_{\mathbf{L}})
  \big]
\Bigg]
= H(\mathbf{L} \mid \mathbf{P}_{\mathbf{L}}) \nonumber \\
& + 
\mathbb{E}_{\mathbf{P}_{\mathbf{L}}}
\!\left[
  \mathrm{KL}\!\left(
    p(\mathbf{L} \mid \mathbf{P}_{\mathbf{L}})
    \,\|\, 
    \phi(\mathbf{L} \mid \mathbf{P}_{\mathbf{L}})
  \right)
\right],
\label{decomp}
\end{align}
}
where $H(\cdot)$ denotes entropy and $\mathrm{KL}(\cdot\|\cdot)$ denotes the Kullback--Leibler divergence (KL). Combining Eqs. \eqref{eq:self_theory}, \eqref{decomp}, 
\begin{equation}   
\label{decomp2}
\mathcal{L}_{\text{self-supervised}}
=
H(\mathbf{L}\mid \mathbf{P}_{\mathbf{L}})
+
\mathbb{E}_{\mathbf{P}_{\mathbf{L}}}
\!\left[
\mathrm{KL}\!\left(
p(\mathbf{L}\mid \mathbf{P}_{\mathbf{L}})
\,\|\, 
\phi(\mathbf{L}\mid \mathbf{P}_{\mathbf{L}})
\right)
\right].
\end{equation}
Since the KL divergence is nonnegative and vanishes if and only if $\phi(\mathbf{L}\mid \mathbf{P}_{\mathbf{L}})
= p(\mathbf{L}\mid \mathbf{P}_{\mathbf{L}})$ almost surely, minimizing $\mathcal{L}_{\text{self-supervised}}$ over $\phi$ 
eliminates the KL term. Therefore,
\begin{equation}
\min_{\phi} \mathcal{L}_{\text{self-supervised}}
=
H(\mathbf{L} \mid \mathbf{P}_{\mathbf{L}}).
\end{equation}
By definition,
\begin{equation}
I(\mathbf{P}_{\mathbf{L}};\mathbf{L})
=
H(\mathbf{L})
-
H(\mathbf{L} \mid \mathbf{P}_{\mathbf{L}})
=
H(\mathbf{L})
-
\min_{\phi} \mathcal{L}_{\text{self-supervised}}.
\end{equation}
Thus, minimizing $\mathcal{L}_{\text{self-supervised}}$ maximizes the mutual information $I(\mathbf{P_L}; \mathbf{L})$ between $\mathbf{P_L}$ and $\mathbf{L}$, and $\mathbf{P_L}$ carries the information needed to predict $\mathbf{L}$ or any layer-specific target, as expected from the causal component of each layer. 

If the predictor $\phi$ is restricted and cannot represent the true conditional distribution 
$p(\mathbf{L}\mid \mathbf{P}_{\mathbf{L}})$ exactly, the KL term may not vanish. In this case, by Eq. \eqref{decomp2}, we have
$
\min_{\phi} \mathcal{L}_{\text{self-supervised}}
\ge
H(\mathbf{L}\mid \mathbf{P}_{\mathbf{L}}),
$
and therefore
\begin{equation}
H(\mathbf{L})
-
\min_{\phi} \mathcal{L}_{\text{self-supervised}}
\le
H(\mathbf{L}) - H(\mathbf{L} \mid \mathbf{P_L}) =
I(\mathbf{P}_{\mathbf{L}};\mathbf{L}).
\end{equation}
Hence, with a restricted $\phi$, minimizing $\mathcal{L}_{\text{self-supervised}}$ corresponds to maximizing 
a lower bound on $I(\mathbf{P_L}; \mathbf{L})$.

\end{proof}

\begin{proposition}
\label{Theorem2}
\end{proposition}

\begin{proof}    
$\mathcal{L}_{\text{matching}}$ and $\mathcal{L}_{\text{causal}}$ push the common embeddings toward representations for which the conditional distribution of $\mathbf{L}$ given the common embedding is uniform, removing layer-specific information from the shared representation.

At any global minimum of $\mathcal{L}_{\mathrm{matching}}$, under the assumption that the encoder function used to compute the common embeddings is sufficiently expressive, we have
$C_\ell = S$ almost surely for all $\ell$,
implying that $\mathbf{C_L}$ is almost surely constant and therefore independent of $\mathbf{L}$. Consequently,
\begin{equation}    
p(\mathbf{L} \mid \mathbf{C_L}) = p(\mathbf{L}), \quad \text{where} \quad \mathbf{L} \sim \text{Unif}\{1,..., N\}.
\end{equation}
This implies that predictions of the layer index based solely on the common embeddings are uniformly distributed, and conditioning on $\mathbf{C_L}$ does not alter the distribution of $\mathbf{L}$.

We can likewise use $\mathcal{L}_{\text{causal}}$ to show that, at optimality, 
$\mathbf{L}$ and the stratified common embedding $\mathbf{C}_{\mathbf{L}'}$ 
become conditionally independent given $\mathbf{P}_{\mathbf{L}}$. 
In $\mathcal{L}_{\text{causal}}$, the common embedding is sampled independently of both 
$\mathbf{L}$ and $\mathbf{P}_{\mathbf{L}}$ by construction. Equivalently, the sampling procedure for the causal loss induces the following joint distribution, capturing the stratification:
\begin{equation}
\label{eq:pairing_factorization}
p(\mathbf{L},\mathbf{P}_{\mathbf{L}},\mathbf{C}_{\mathbf{L}'})
=
p(\mathbf{L},\mathbf{P}_{\mathbf{L}})\,p(\mathbf{C}_{\mathbf{L}'}),
\qquad \mathbf{L}'\perp \mathbf{L}.
\end{equation}
If we compute $p(\mathbf{L}\mid \mathbf{P}_{\mathbf{L}},\mathbf{C}_{\mathbf{L}'})$, by definition,
\begin{equation}
p(\mathbf{L}\mid \mathbf{P}_{\mathbf{L}},\mathbf{C}_{\mathbf{L}'})
=
\frac{p(\mathbf{L},\mathbf{P}_{\mathbf{L}},\mathbf{C}_{\mathbf{L}'})}
     {p(\mathbf{P}_{\mathbf{L}},\mathbf{C}_{\mathbf{L}'})}.
     \label{eq_conditional}
\end{equation}
The denominator can be written as
{\fontsize{9}{8}\selectfont
\begin{align}
p(\mathbf{P}_{\mathbf{L}},\mathbf{C}_{\mathbf{L}'})
&=
\sum_{\ell=1}^{N} p(\mathbf{L}=\ell,\mathbf{P}_{\mathbf{L}},\mathbf{C}_{\mathbf{L}'}) \notag
\overset{\eqref{eq:pairing_factorization}}{=}
\sum_{\ell=1}^{N} p(\mathbf{L}=\ell,\mathbf{P}_{\mathbf{L}})\,p(\mathbf{C}_{\mathbf{L}'}) \notag \\
&=
p(\mathbf{C}_{\mathbf{L}'})
\sum_{\ell=1}^{N} p(\mathbf{L}=\ell,\mathbf{P}_{\mathbf{L}}) 
=
p(\mathbf{C}_{\mathbf{L}'})\,p(\mathbf{P}_{\mathbf{L}}),
\label{eq_denom}
\end{align}
}
where we use the fact that $p(\mathbf{C}_{\mathbf{L}'})$ does not depend on $\ell$.

By combining Eqs. \eqref{eq_conditional} and \eqref{eq_denom}, we have
\begin{align}
\label{eq:cond_indep_step}
p(\mathbf{L}\mid \mathbf{P}_{\mathbf{L}},\mathbf{C}_{\mathbf{L}'})
&=
\frac{p(\mathbf{L},\mathbf{P}_{\mathbf{L}})\,p(\mathbf{C}_{\mathbf{L}'})}
     {p(\mathbf{P}_{\mathbf{L}})\,p(\mathbf{C}_{\mathbf{L}'})}
=
p(\mathbf{L}\mid \mathbf{P}_{\mathbf{L}}).
\end{align}
Hence, the stratified pairing procedure, used to implement the backdoor-style adjustment, enforces conditional independence between $\mathbf{L}$ and $\mathbf{C}_{\mathbf{L}'}$ 
given $\mathbf{P}_{\mathbf{L}}$.

The causal loss in Eq. \eqref{eq:causal_theory} has the same conditional cross-entropy form as the self-supervised loss in Eq. \eqref{eq:self_theory}, analyzed in Proposition \eqref{theorem1}. By the same entropy--KL decomposition as above in Eq. \eqref{decomp2}, we obtain

{\fontsize{8}{8}\selectfont
\begin{equation}    
\mathcal{L}_{\text{causal}}
=
H(\mathbf{L}\mid 
\mathbf{P}_{\mathbf{L}} \oplus \mathbf{C}_{\mathbf{L}'})
+
\mathbb{E}\!\left[
\mathrm{KL}\!\left(
p(\mathbf{L}\mid 
\mathbf{P}_{\mathbf{L}} \oplus \mathbf{C}_{\mathbf{L}'})
\,\|\, 
\psi(\mathbf{L}\mid 
\mathbf{P}_{\mathbf{L}} \oplus \mathbf{C}_{\mathbf{L}'})
\right)
\right].
\end{equation}
}
Hence, if the predictor class for $\psi$ is sufficiently expressive, the minimum of $\mathcal{L}_{\text{causal}}$ is achieved when the KL term vanishes, yielding the Bayes-optimal predictor: 
\begin{align}
\psi^{*}\!\left(\mathbf{L}\mid \mathbf{P}_{\mathbf{L}} \oplus \mathbf{C}_{\mathbf{L}'}\right)
&\overset{}{=} 
p\!\left(\mathbf{L}\mid \mathbf{P}_{\mathbf{L}} \oplus \mathbf{C}_{\mathbf{L}'}\right) \notag \\
&\overset{(a)}{=}
p\!\left(\mathbf{L}\mid \mathbf{P}_{\mathbf{L}}, \mathbf{C}_{\mathbf{L}'}\right) 
\overset{\eqref{eq:cond_indep_step}}{=}
p\!\left(\mathbf{L}\mid \mathbf{P}_{\mathbf{L}}\right).
\label{psiopt}
\end{align}
where (a) follows from the fact that if $\oplus$ is an injective function such as concatenation, \(\mathbf{P}_{\mathbf{L}}\) can be recovered from \(\mathbf{P}_{\mathbf{L}} \oplus \mathbf{C}_{\mathbf{L}'}\) via a deterministic map (e.g., projection) and therefore conditioning on \(\mathbf{P}_{\mathbf{L}} \oplus \mathbf{C}_{\mathbf{L}'}\) is equivalent to conditioning on \((\mathbf{P}_{\mathbf{L}}, \mathbf{C}_{\mathbf{L}'})\). 
Eq. \eqref{psiopt} shows that, under the stratified pairing construction, the optimal predictor $\psi^\star$ depends only on 
$\mathbf{P}_{\mathbf{L}}$ and ignores $\mathbf{C}_{\mathbf{L}'}$. In particular, 
$\mathbf{C}_{\mathbf{L}'}$ carries no additional layer-specific information for predicting the layer index beyond $\mathbf{P}_{\mathbf{L}}$, while $\mathbf{P}_{\mathbf{L}}$ retains the information necessary to predict $\mathbf{L}$.

\end{proof}

\begin{theorem}
\label{theorem3}
\end{theorem}

\begin{proof}[Proof sketch]
We first clarify our probabilistic notation again. Throughout the analysis, bold symbols denote random variables, while non-bold symbols denote their realizations. In particular, $\mathbf{U}$ denotes the shared latent random variable, $\mathbf{V}_{\mathbf{L}}$ the layer-specific 
latent random variable associated with a randomly selected layer 
$\mathbf{L}$, and $\mathbf{Z}_{\mathbf{L}}$ the observable random variable corresponding to that layer. Their realizations are written 
as $U$, $V_\ell$, and $Z_\ell$, respectively. 

The random variable $\mathbf{Z}_{\mathbf{L}}$ represents all observable information in a randomly selected layer. For a fixed layer index $\mathbf{L}=\ell$, its realization is denoted by
$
    Z_{\ell} = (G^{\ell}, X^{\ell}),
$
where $G^{(\ell)}$ and $X^{(\ell)}$ denote the graph structure and node features of layer $\ell$, respectively.

We assume that the layer-specific latent variables and noise terms are independent across layers and that the shared latent factor and the layer-specific latent factor are independent: $\mathbf{U} \perp \mathbf{V}_{\mathbf{L}}$. In particular, conditional on $\mathbf{U}$, variation in $\mathbf{V}_{\mathbf{L}}$ affects only layer-specific aspects of the observed data. We also assume that the shared factor is independent of the random layer index:\hspace{1pt}$\mathbf{U} \perp \mathbf{L}$.

In practice, common and private embeddings are computed via GNN encoders and MLPs, which we abstract in the theoretical analysis as measurable deterministic mappings:
$
\label{eqfunc}
C_\ell = f_C(Z_\ell)
$, $
P_\ell = f_P(Z_\ell).
$
This is justified because each encoder is ultimately a deterministic function of %the layer data 
$Z_\ell$.

We now present theoretical results showing that, under the latent-factor model, the proposed losses recover disentangled common and private embeddings.

\newtheorem{subtheorem}{Theorem}[theorem]
\renewcommand{\thesubtheorem}{\thetheorem.\arabic{subtheorem}}
\begin{subtheorem}
\label{disC}
At any global minimum of $\mathcal{L}_{\text{matching}}$ in Eq.~\eqref{eq2:matching}, $\mathbf{C}_{\mathbf{L}} = f_C(\mathbf{Z}_{\mathbf{L}})$ is almost surely a measurable function of the shared latent random variable $\mathbf{U}$ alone. 
\end{subtheorem}
\begin{proof}
Let $\mathbf{L}$ and $\mathbf{L}'$ be independent and identically distributed copies of the layer index. For the theoretical analysis, it is convenient to rewrite the matching loss in Eq. \eqref{eq2:matching} in the following form:
\begin{equation}\label{eq:Lmatch_clean}
\mathcal{L}_{\text{matching}}(f_C)
:=
\mathbb{E}_{\mathbf{U},\,
\mathbf{V}_{\mathbf{L}},\,\mathbf{V}_{\mathbf{L}'},\,
\boldsymbol{\varepsilon}_{\mathbf{L}},\,
\boldsymbol{\varepsilon}_{\mathbf{L}'}}
\!\left[
  \left\|
  f_C(\mathbf{Z}_{\mathbf{L}})
  -
  f_C(\mathbf{Z}_{\mathbf{L}'})
  \right\|^{2}
\right].
\end{equation}
We assume that the encoder class for $f_C$ is sufficiently expressive to approximate any square-integrable measurable function of $(\mathbf{U}, \mathbf{V}_{\mathbf{L}}, \boldsymbol{\varepsilon}_{\mathbf{L}})
$. This assumption is justified by the universal approximation capacity of GNNs \cite{NEURIPS2020_e4acb4c8} followed by MLPs \cite{10.5555/70405.70408}, which we use to parameterize $f_C$.

Applying the law of total expectation and conditioning on the shared latent
variable $\mathbf{U}$ yields
\begin{equation}\label{eq:Lmatch_cond_clean}
\mathcal{L}_{\text{matching}}(f_C)
=
\mathbb{E}_{\mathbf{U}}
\!\left[
  \mathbb{E}_{
  \mathbf{V}_{\mathbf{L}},\,\mathbf{V}_{\mathbf{L}'},\,
  \boldsymbol{\varepsilon}_{\mathbf{L}},\,
  \boldsymbol{\varepsilon}_{\mathbf{L}'}
  \mid \mathbf{U}}
  \!\left[
    \left\|
    \mathbf{C}_{\mathbf{L}}
    -
    \mathbf{C}_{\mathbf{L}'}
    \right\|^{2}
  \right]
\right].
\end{equation}
Minimizing $\mathcal{L}_{\text{matching}}$ over $f_C$
is therefore equivalent to minimizing the inner conditional expectation for almost every realization $\mathbf{U}=U$. For a fixed realization $U$, 
%the conditional random variable 
$\mathbf{C}_{\mathbf{L}}^{(U)}
:=
\mathbf{C}_{\mathbf{L}} \mid \mathbf{U}=U
$ depends only on the remaining randomness
$(\mathbf{V}_{\mathbf{L}}, \boldsymbol{\varepsilon}_{\mathbf{L}})$. The conditional objective becomes
\begin{equation}
\mathcal{L}_U(f_C)
=
\mathbb{E}_{\,
\mathbf{V}_{\mathbf{L}},\,\mathbf{V}_{\mathbf{L}'},\,
\boldsymbol{\varepsilon}_{\mathbf{L}},\,
\boldsymbol{\varepsilon}_{\mathbf{L}'}
\,\mid\, \mathbf{U}=U}
\!\left[
  \left\|
  \mathbf{C}_{\mathbf{L}}^{(U)}
  -
  \mathbf{C}_{\mathbf{L}'}^{(U)}
  \right\|^{2}
\right].
\end{equation}
Under the richness assumption on the function class of $f_C$, for fixed $U$, it can represent any square-integrable measurable function of $(\mathbf{V}_{\mathbf{L}}, \boldsymbol{\varepsilon}_{\mathbf{L}})$. 
Since the integrand is nonnegative, we have $\mathcal{L}_U(f_C) \ge 0$. Moreover, the infimum of $\mathcal{L}_U(f_C)$ over all admissible $f_C$ is zero, and this infimum is attainable whenever
\begin{equation}    
\mathbf{C}_{\mathbf{L}}^{(U)}
=
\mathbf{C}_{\mathbf{L}'}^{(U)}
\quad
\text{almost surely.}
\end{equation}

If this equality fails on a set of positive conditional probabilities, then the squared norm contributes strictly positively to the expectation,
i.e., $\mathcal{L}_U(f_C) > 0$.
Therefore, any global minimizer must satisfy $\mathbf{C}_{\mathbf{L}}^{(U)}
=
\mathbf{C}_{\mathbf{L}'}^{(U)}$ almost surely.
So the common representation does not vary with the remaining randomness (layer-specific latent factor and noise) once $\mathbf{U}$ is fixed. Equivalently, the conditional variance vanishes:
$
    \mathrm{Var}\!\left(\mathbf{C_L}\mid \mathbf{U}\right)=0
$ almost surely.

It is a standard result in probability theory that a random variable $\mathbf{Y}$ satisfies $\mathrm{Var}(\mathbf{Y}\mid \mathbf{U})=0$ almost surely
if and only if $\mathbf{Y}$ is measurable with respect to the $\sigma$-algebra generated by $\mathbf{U}$, denoted $\sigma(\mathbf{U})$. In other words, $\mathbf{Y}$ is almost surely a function of $\mathbf{U}$ \cite{kallenberg2002foundations}. By the Doob--Dynkin lemma \cite{kallenberg2002foundations} \cite{key0058896m}, any $\sigma(\mathbf{U})$-measurable random variable can be written as a measurable function of $\mathbf{U}$. Therefore, there exists a measurable function $h$ such that
$\mathbf{C_L}=h(\mathbf{U})$ almost surely.
Hence, the common embedding $\mathbf{C_L}$ depends only on the shared latent random variable $\mathbf{U}$. 

To summarize, the matching loss enforces agreement of the realizations $C_\ell$ across layers for the same underlying object characterized by $\mathbf{U}$. Since layers differ only through the layer-specific latent variables $\mathbf{V_L}$ and noise variables $\boldsymbol{\varepsilon_L}$, and the encoder $f_C$ is sufficiently expressive to represent any square-integrable measurable function of its inputs, any minimizer achieving zero matching loss must eliminate all dependence on $(\mathbf{V_L}, \boldsymbol{\varepsilon_L})$ and retain dependence solely on $\mathbf{U}$. Thus, the optimal common embedding is almost surely a function of the shared latent variable $\mathbf{U}$ alone.

\end{proof}

\begin{subtheorem}
\label{dicP}
Under the global optimality of
$\mathcal{L}_{\text{self-supervised}}$ in Eq.~\eqref{eq:self_theory}, the optimal private embedding $\mathbf{P}_{\mathbf{L}}^\star$ is sufficient for predicting the layer index $\mathbf{L}$ with respect to the observed data $\mathbf{Z}_{\mathbf{L}}$. If, in addition, $\mathbf{P}_{\mathbf{L}}^\star$ is (approximately) minimal sufficient for $\mathbf{L}$, an assumption encouraged in practice by optimizing $\mathcal{L}_{\text{causal}}$ in Eq.~\eqref{eq:causal_theory} and $\mathcal{L}_{\text{self-supervised}}$ in Eq.~\eqref{eq:self_theory} together with the limited capacity of the model, then the optimal private embedding 
$\mathbf{P}^*_{\mathbf{L}}$ is almost surely a measurable function of the layer-specific latent variable $\mathbf{V}_{\mathbf{L}}$ alone.

\end{subtheorem}

\begin{proof}
We assume that the layer index is a measurable function of the 
layer-specific latent:
$
\mathbf{L} = \Omega(\mathbf{V}_{\mathbf{L}}),
$
so that the $\sigma$-algebras satisfy $\sigma(\mathbf{L}) \subseteq \sigma(\mathbf{V}_{\mathbf{L}}).
$
Intuitively, $\mathbf{V}_{\mathbf{L}}$ captures all layer-specific information and determines $\mathbf{L}$, while $\mathbf{U}$ does not influence $\mathbf{L}$.
As a result, from the joint factorization
\begin{equation}
p(\mathbf{U}, \mathbf{V}_{\mathbf{L}}, 
\mathbf{L}, \mathbf{Z}_{\mathbf{L}})
=
p(\mathbf{U})\,
p(\mathbf{V}_{\mathbf{L}})\,
p(\mathbf{L}\mid \mathbf{V}_{\mathbf{L}})\,
p(\mathbf{Z}_{\mathbf{L}}\mid 
\mathbf{U}, \mathbf{V}_{\mathbf{L}}).
\end{equation}
It follows that, conditional on 
$\mathbf{V}_{\mathbf{L}}$, 
the layer index 
$\mathbf{L}$ 
and the observation 
$\mathbf{Z}_{\mathbf{L}}$ 
are independent, i.e. $\mathbf{L} \perp \mathbf{Z}_{\mathbf{L}} \mid \mathbf{V}_{\mathbf{L}}
$, 
or equivalently, $p(\mathbf{L} \mid \mathbf{Z}_{\mathbf{L}}, 
\mathbf{V}_{\mathbf{L}})
=
p(\mathbf{L} \mid \mathbf{V}_{\mathbf{L}}).
$ Thus, the latent variable 
$\mathbf{V}_{\mathbf{L}}$ 
is a sufficient statistic for 
$\mathbf{L}$ with respect to 
$\mathbf{Z}_{\mathbf{L}}$.

The private encoder is a measurable map $\mathbf{P}_{\mathbf{L}}
%\overset{\eqref{fp-fc}}{=}
=
f_P(\mathbf{Z}_{\mathbf{L}}).
$
We say that the optimal $\mathbf{P}_{\mathbf{L}}$ ($\mathbf{P}^*_{\mathbf{L}}$) is sufficient for 
$\mathbf{L}$ with respect to 
$\mathbf{Z}_{\mathbf{L}}$ if $p(\mathbf{L}\mid \mathbf{Z}_{\mathbf{L}}, \mathbf{P}^*_\mathbf{L})
=
p(\mathbf{L}\mid \mathbf{P}^*_{\mathbf{L}}) $ almost surely, or equivalently $\mathbf{L}
\perp
\mathbf{Z}_{\mathbf{L}}
\mid
\mathbf{P}^*_{\mathbf{L}}.
$
Intuitively, once $\mathbf{P}^*_{\mathbf{L}}$ is known, the full observation $\mathbf{Z}_{\mathbf{L}}$ provides no additional information for predicting $\mathbf{L}$.
This follows from the two nested
minimization in $\mathcal{L}_{\text{self-supervised}}$ in Eq. \eqref{eq:self_theory}: (1) Minimize over the head 
  $\phi$ 
  for a fixed private encoder $f_P$,
  and (2) Minimize over $f_P$ itself.
Thus, we can reformulate 
$\mathcal{L}_{\text{self-supervised}}$ as follows:

{\small
\begin{equation}\label{eq:Lpriv-def}
\mathcal L_{\text{self-supervised}}(f_P,\phi)
=
\mathbb{E}_{\mathbf{L},\,\mathbf{P}_{\mathbf{L}}}
\big[
-\log \phi(\mathbf{L}\mid \mathbf{P}_{\mathbf{L}})
\big],
\quad
\mathbf{P}_{\mathbf{L}}
=
f_P(\mathbf{Z}_{\mathbf{L}}).
\end{equation}
}
We define the global infimum
\begin{equation}\label{eq:global-inf}
\mathcal L_{\text{self-supervised}}^\star
:=
\inf_{f_P,\phi}\,
\mathcal L_{\text{self-supervised}}(f_P,\phi).
\end{equation}
To characterize $\mathcal L_{\text{self-supervised}}^\star$, we analyze the minimization levels in two steps.

\textbf{Step 1: Minimization over the prediction head $\phi$.} For a fixed encoder $f_P$, $\mathbf{P}_{\mathbf{L}}$ is fixed as a deterministic function of $\mathbf{Z}_{\mathbf{L}}$. Optimizing over all measurable heads $\phi$, Proposition~\ref{theorem1} implies that $\mathcal{L}_{\text{self-supervised}}$ is minimized when
\begin{equation}\label{eq:bayes-head}
\phi^\star(\mathbf{L}\mid \mathbf{P}_{\mathbf{L}})
=
p(\mathbf{L}\mid \mathbf{P}_{\mathbf{L}}).
\end{equation}
Plugging Eq. \eqref{eq:bayes-head} into Eq. \eqref{eq:Lpriv-def} gives
{\small
\begin{align}
    \label{eq:step1-min-proof}
\min_{\phi}\,
\mathcal L_{\text{self-supervised}}(f_P,\phi)
=
&\mathcal L_{\text{self-supervised}}(f_P,\phi^\star)
\\
=&\mathbb{E}_{\mathbf{L},\,\mathbf{P}_{\mathbf{L}}}
\big[
-\log p(\mathbf{L}\mid \mathbf{P}_{\mathbf{L}})
\big]
=
H(\mathbf{L}\mid \mathbf{P}_{\mathbf{L}}).
\end{align}
}

\textbf{Step 2: Minimization over the encoder $f_P$.} 
Combining \eqref{eq:global-inf} and \eqref{eq:step1-min-proof},
\begin{align}
\label{eq:global-inf-rewrite}
\mathcal L_{\text{self-supervised}}^\star
&=
\inf_{f_P}\,
\inf_{\phi}\,
\mathcal L_{\text{self-supervised}}(f_P,\phi)\\
&=
\inf_{f_P}
H\!\left(
\mathbf{L}\mid \mathbf{P}_{\mathbf{L}}
\right),
\qquad
\mathbf{P}_{\mathbf{L}}
=
f_P(\mathbf{Z}_{\mathbf{L}}).
\end{align}
Thus, globally, the optimization reduces to minimizing the conditional entropy  $H(\mathbf{L}\mid f_P(\mathbf{Z}_{\mathbf{L}}))$ over all encoders $f_P$. We use the fact that, for a fixed $f_P$, $\mathbf{P_L}$ is a deterministic function of $\mathbf{Z_L}$, together with the chain rule for conditional mutual information \cite{Cover2006}:
\begin{equation}\label{eq:cond-mi-identity-proof}
H(\mathbf{L}\mid \mathbf{P}_{\mathbf{L}})
=
H(\mathbf{L}\mid \mathbf{Z}_{\mathbf{L}})
+
I\!\left(
\mathbf{L};\mathbf{Z}_{\mathbf{L}}\mid \mathbf{P}_{\mathbf{L}}
\right).
\end{equation}
Since $I(\mathbf{L};\mathbf{Z}_{\mathbf{L}}\mid \mathbf{P}_{\mathbf{L}})\ge 0$, we have $H(\mathbf{L}\mid \mathbf{P}_{\mathbf{L}})
\ge
H(\mathbf{L}\mid \mathbf{Z}_{\mathbf{L}})
$. If the encoder class for $f_P$ is sufficiently rich to represent a sufficient statistic of $\mathbf{Z}_{\mathbf{L}}$ for $\mathbf{L}$, then the lower bound is attainable, and thus
\begin{equation}\label{eq:attain-bound-proof}
\inf_{f_P} H(\mathbf{L}\mid \mathbf{P}_{\mathbf{L}})
=
H(\mathbf{L}\mid \mathbf{Z}_{\mathbf{L}}).
\end{equation}

Let $f_P^\star$ be a global minimizer achieving \eqref{eq:attain-bound-proof}, and let
$\mathbf{P}_{\mathbf{L}}^\star := f_P^\star(\mathbf{Z}_{\mathbf{L}})$. Then $H(\mathbf{L}\mid \mathbf{P}_{\mathbf{L}}^\star)
=
H(\mathbf{L}\mid \mathbf{Z}_{\mathbf{L}}).
$ Plugging this equality into \eqref{eq:cond-mi-identity-proof} forces $I\!\left(
\mathbf{L};\mathbf{Z}_{\mathbf{L}} \mid \mathbf{P}_{\mathbf{L}}^\star
\right)=0,
$ which is equivalent to the conditional independence
$\mathbf{L}\perp \mathbf{Z}_{\mathbf{L}} \mid \mathbf{P}_{\mathbf{L}}^\star$.
This shows that
$\mathbf{P}_{\mathbf{L}}^\star$ is sufficient for $\mathbf{L}$ with respect to $\mathbf{Z}_{\mathbf{L}}$.

Among all sufficient statistics 
$T(\mathbf{Z}_{\mathbf{L}})$ for $\mathbf{L}$, $\mathbf{P}_{\mathbf{L}}^\star$ 
is minimal sufficient if, for any other sufficient statistic 
$T(\mathbf{Z}_{\mathbf{L}})$, there exists a measurable function $f$ such that $T(\mathbf{Z}_{\mathbf{L}})
=
f(\mathbf{P}_{\mathbf{L}}^\star)
\quad \text{almost surely}.
$
Equivalently, the $\sigma$-algebra $\sigma(\mathbf{P}^*_{\mathbf{L}})$ is the smallest $\sigma$-algebra making $\mathbf{L}$ independent of $\mathbf{Z}_{\mathbf{L}}$ given $\mathbf{P}^*_{\mathbf{L}}$.

In practice, we do not explicitly optimize for minimality, but several aspects of the CaDeM training objective encourage the private embedding $\mathbf{P}^*_{\mathbf{L}}$ to approach a minimal sufficient encoding. $\mathbf{P}_{\mathbf{L}}$ has fixed dimension $d$ and is produced by a finite-capacity network. Since $\mathcal{L}_{\text{self-supervised}}$ and $\mathcal{L}_{\text{causal}}$ train $\mathbf{P}_{\mathbf{L}}$ to predict $\mathbf{L}$, allocating capacity to encode information irrelevant for predicting $\mathbf{L}$ does not improve the training objectives and is therefore discouraged during optimization. In addition, $\mathcal{L}_{\text{matching}}$ aligns $\mathbf{C_L}$ across layers, encouraging shared information encoded in $\mathbf{U}$ to be captured by $\mathbf{C_L}$ (Theorem \ref{disC}) rather than leaking into $\mathbf{P_L}$.
These considerations induce an information-bottleneck type bias \cite{articleee, ShwartzZiv2017OpeningTB, 8503149}: among the many embeddings that are sufficient for $\mathbf{L}$, the training dynamics prefer those that are as small and focused on $\mathbf{L}$ as possible. 
In our theoretical results below, we therefore assume that 
$\mathbf{P}_{\mathbf{L}}^\star$ 
is (approximately) minimal sufficient for 
$\mathbf{L}$.

Let $\mathcal{A} := \sigma(\mathbf{P}_{\mathbf{L}}^\star)
\quad \text{and} \quad
\mathcal{B} := \sigma(\mathbf{V}_{\mathbf{L}})
$ be the $\sigma$-algebras generated by 
$\mathbf{P}_{\mathbf{L}}^\star$ 
and 
$\mathbf{V}_{\mathbf{L}}$, respectively. Since $\mathbf{P}_{\mathbf{L}}^\star$ is sufficient for 
$\mathbf{L}$ with respect to $\mathbf{Z_L}$, 
$\mathcal{A}$ is a sufficient $\sigma$-algebra:
$
\mathbf{L}
\perp
\mathbf{Z}_{\mathbf{L}}
\mid
\mathcal{A}.
$ Similarly, because $\mathbf{V}_{\mathbf{L}}$ is sufficient for 
$\mathbf{L}$ with respect to the $\mathbf{Z_L}$, 
$\mathcal{B}$ is also sufficient $\sigma$-algebra:
$
\mathbf{L}
\perp
\mathbf{Z}_{\mathbf{L}}
\mid
\mathcal{B}.
$ Since $\mathbf{P}_{\mathbf{L}}^\star$ is minimal sufficient, 
$\mathcal{A}$ is contained in any other sufficient 
$\sigma$-algebra, in particular in $\mathcal{B}$:
$
\mathcal{A} \subseteq \mathcal{B}.
$ This inclusion 
$
\sigma(\mathbf{P}_{\mathbf{L}}^\star)
\subseteq
\sigma(\mathbf{V}_{\mathbf{L}})
$
implies that 
$\mathbf{P}_{\mathbf{L}}^\star$ 
is measurable with respect to 
$\mathbf{V}_{\mathbf{L}}$ 
\cite{kallenberg2002foundations, key0058896m}. 
Therefore, there exists a measurable function $g$ such that
$
\mathbf{P}_{\mathbf{L}}^\star
=
g(\mathbf{V}_{\mathbf{L}})$ almost surely.
Thus, $\mathbf{P}_{\mathbf{L}}^\star$ depends only on the layer-specific latent random variable.

\end{proof}

\begin{subtheorem}
\label{dispc}
Under the assumptions of Theorems~\ref{disC} and~\ref{dicP}, $\mathbf{C}_{\mathbf{L}}$ and $\mathbf{P}_{\mathbf{L}}$ are disentangled in the sense that each encodes only role-specific information.
\end{subtheorem}

\begin{proof}
By Theorems~\ref{disC} and~\ref{dicP}, at any global optimum of the objectives in Eqs. \eqref{eq2:matching}, \eqref{eq:self_theory}, and \eqref{eq:causal_theory}, $\mathbf{C}_{\mathbf{L}}$  and $\mathbf{P}_{\mathbf{L}}$ are measurable functions of the shared latent $\mathbf{U}$, and layer-specific latent $\mathbf{V_L}$ respectively.
The roles of the two components are therefore disentangled: $\mathbf{C}_{\mathbf{L}}$ depends only on the shared latent variable $\mathbf{U}$ and encodes no information beyond $\mathbf{U}$; $\mathbf{P}_{\mathbf{L}}$ depends only on the layer-specific latent variable $\mathbf{V}_{\mathbf{L}}$ and encodes only information about $\mathbf{V}_{\mathbf{L}}$; Consequently, $\mathbf{C}_{\mathbf{L}}$ does not depend on $\mathbf{V}_{\mathbf{L}}$, and $\mathbf{P}_{\mathbf{L}}$ does not depend on $\mathbf{U}$.

\end{proof}

\textbf{Role of the causal loss in achieving disentanglement.}
Our theoretical results show that, at global optimality, $\mathbf{C}_{\mathbf{L}}$ becomes a measurable function of $\mathbf{U}$ and therefore does not carry layer-specific information. They further establish that $\mathbf{P}_{\mathbf{L}}^\star$ is
sufficient for predicting $\mathbf{L}$.
To conclude that $\mathbf{P}_{\mathbf{L}}^\star$ depends only on $\mathbf{V}_{\mathbf{L}}$, we additionally assume that it is (approximately) minimal sufficient. The $\mathcal{L}_\text{causal}$ in Eq.~\eqref{eq:causal_theory} plays a central role in supporting these properties in practice.

From Proposition~\ref{Theorem2}, we have $\psi^\star(\mathbf{L}\mid
\mathbf{P}_{\mathbf{L}}, \mathbf{C}_{\mathbf{L}'})
=
p(\mathbf{L}\mid \mathbf{P}_{\mathbf{L}}),
$
under the stratified pairing construction.
Thus, at the optimum, the optimal predictor is independent of $\mathbf{C}_{\mathbf{L}'}$, and $\mathcal{L}_\text{causal}$ provides no incentive to encode layer-specific information in the common embeddings. At the same time, it reinforces that layer information must be captured
by $\mathbf{P}_{\mathbf{L}}$. Combined with $\mathcal{L}_\text{matching}$ in Eq.~\eqref{eq2:matching}, which
penalizes layer-dependent variation in $\mathbf{C}_{\mathbf{L}}$, and
with standard capacity or regularization constraints, the overall objective induces an inductive bias: layer-specific information is driven into $\mathbf{P}_{\mathbf{L}}$, while the common branch $\mathbf{C}_{\mathbf{L}}$ is pushed toward layer invariance. This supports the assumption that $\mathbf{P}_{\mathbf{L}}^\star$ behaves as a
(near) minimal sufficient statistic for $\mathbf{L}$ and therefore becomes a measurable function of 
%the layer-specific latent variable 
$\mathbf{V}_{\mathbf{L}}$, while $\mathbf{C}_{\mathbf{L}}$ becomes a measurable function of $\mathbf{U}$.

\end{proof}

\section{Synthetic data generation}
Table~\ref{tab:syn_summary} summarizes the key configuration parameters of the synthetic datasets used in our experiments. Here, 
$M$ denotes the number of nodes, $N$ the number of layers, $N_\text{aug}$ the number of augmentations per layer, $r$ the node sampling ratio in graph augmentation, and $\sigma$ the standard deviation of the Gaussian augmentation noise. Learning rates are reported as a pair: the first corresponds to $\phi$ and $\psi$, and the second to the remaining network parameters.
All synthetic datasets share the same architecture, embedding dimension 
$d=8$, and Adam optimizer, with weight decay 
$10^{-4}$ applied to $\phi$ and $\psi$. Feature dropout is set to $0.1$.

To evaluate \textbf{Syn1} embeddings, we train a linear classifier for node classification over $50$ epochs with a learning rate $0.3$.

To assess disentanglement in \textbf{Syn2}, we apply K-means ($K=3$) to the common and private embeddings and evaluate their performance over $50$ runs using ARI and NMI. 

In \textbf{Syn3}, the scale parameter $t$ used in localized wavelet atoms is set to $\frac{20}{\lambda_{\max}}$, where $\lambda_{\max}$ is the spectral radius of the graph Laplacian estimated via power iteration. The parameter $t$ controls the spatial extent of the atoms, with smaller values producing more localized, high-frequency patterns. Normalizing by $\lambda_{\max}$ makes the scale independent of the Laplacian spectrum. For evaluation, we repeat K-means over $50$ runs and report the mean and standard deviation of ARI and NMI.

In \textbf{Syn4}, we consider three dynamical classes, and each class is instantiated using two distinct nonlinear dynamical systems, yielding six dynamics in total. Details of these dynamics are listed in Table \ref{tab:dynamics}, where $X^\ell_i(t)$ denotes the feature (state) of node $i$ at layer $\ell$ at time $t$, $A_{ij}$ is the $(i,j)$-th entry of the binary adjacency matrix $A$. We rescale $A$ by its spectral radius to obtain $\tilde{A}$ and replace $A$ with $\beta \tilde{A}$ in Table~\ref{tab:dynamics} to control coupling strength and ensure numerical stability, where $\beta$ is a class-dependent scaling parameter (see below). All systems are simulated using fixed-step forward Euler integration for $T$ steps with step size $\Delta t$. For degree-driven systems, we use $T=250$, $\Delta t=0.04$, and $\beta=1.0$; for homogeneous systems, $T=300$, $\Delta t=0.04$, and $\beta=1.2$; and for degree-avert systems, $T=100$, $\Delta t=0.02$, and $\beta=0.2$. Node states $X_0 \in \mathbb{R}^{M\times 1}$ are initialized independently for each layer using dynamic-specific random distributions: Gaussian $\mathcal{N}(0,0.5^2)$ for population dynamics, uniform $[0,0.5]$ for regulatory (degree-driven), uniform $[0,0.4]$ for epidemic, $\mathrm{Beta}(2,2)$ for biochemical, and uniform $[0,0.6]$ for mutualistic and regulatory (degree-avert) dynamics. 
Finally, to evaluate, a graph-level classifier is trained for $400$ epochs with a learning rate of $0.1$. As an additional analysis, we apply t-SNE to the private and common graph embeddings and visualize them across the three- and six-class settings to examine their behavior. The resulting plots are shown in Fig. \ref{figsyn4}.  For the common embeddings (Fig. \ref{figsyn4}(a)), points from different classes are mixed, indicating that class-specific information is not encoded, as intended. In contrast, the private embeddings for layer 1 and layer 2 (Fig. \ref{figsyn4}(b), \ref{figsyn4}(c)) form three well-separated clusters, each corresponding to a dynamical class, showing that they capture layer-specific dynamics. Moreover, the private graph embeddings for the six subclasses (Fig. \ref{figsyn4}(d)) cluster into six distinct groups, indicating that the representations preserve finer-grained subclass distinctions.

\begin{table*}[t]
\centering
\caption{Summary of synthetic datasets and experimental configurations.}
\label{tab:syn_summary}
\small
\setlength{\tabcolsep}{2.5pt}
\renewcommand{\arraystretch}{0.8}
\begin{tabular}{lcccccccccc}
\toprule
Dataset 
& $M$ 
& $N$ 
& Task 
& \#Clusters 
& $N_{\text{aug}}$ 
& $r$ 
& $\sigma$ 
& Epochs 
& Learning rates 
& Loss weights $(\mathcal{L}_{\text{match}}, \mathcal{L}_{\text{self-sup}},\mathcal{L}_{\text{causal}})$ \\
\midrule

Syn1 
& $100$ 
& $3$ 
& Node classification 
& $K=3$ 
& $5$ 
& $0.6$ 
& $0.1$ 
& $140$ 
& $(10^{-2},\,10^{-4})$ 
& $(0.9,\,1.5,\,3.4)$ \\

Syn2 
& $1000$ 
& $3$ 
& Node clustering 
& $K=3$ 
& $20$ 
& $0.6$ 
& $0.1$ 
& $200$ 
& $(10^{-2},\,10^{-4})$ 
& $(1, \,0.5,\,0.5)$ \\

Syn3 
& $600$ 
& $2$ 
& Node clustering 
& $3$ centers/layer 
& $20$ 
& $0.6$ 
& $0.1$ 
& $200$ 
& $(10^{-2},\,10^{-4})$ 
& $(1, \, 0.05,\,0.05)$ \\

Syn4 
& $100$ 
& $2$ 
& Graph classification 
& $3$ classes (or $6$ subclasses) 
& $20$ 
& $0.6$ 
& $0.1$ 
& $200$ 
& $(10^{-2},\,10^{-4})$ 
& $(1, \, 0.5,\,0.5)$ \\

\bottomrule
\end{tabular}
\end{table*}

\newcolumntype{C}[1]{>{\centering\arraybackslash}p{#1}}
\begin{table}[!!!!!!!!!!!!!!!!!!!!!t]
\centering
\Large
%\footnotesize
\caption{Summary of dynamical systems used to generate layer-specific signals in Syn4. 
%Each dynamic governs the temporal evolution of node states.
}
\label{tab:dynamics}
\setlength{\tabcolsep}{2pt}
\renewcommand{\arraystretch}{1.1}
\resizebox{\columnwidth}{!}{
\begin{tabular}{l C{13cm}}
\toprule
\textbf{Dynamics} & \textbf{Equation} \\
\midrule

Population dynamics &
$\displaystyle
\frac{dX_i^\ell(t)}{dt}
= -\bigl(X_i^\ell(t)\bigr)^3
+ \sum_{j=1}^{M} A_{ij}\,\bigl(X_j^\ell(t)\bigr)^2
$ \\

Regulatory dynamics (degree-driven) &
$\displaystyle
\frac{dX_i^\ell(t)}{dt}
= -X_i^\ell(t)
+ \sum_{j=1}^{M} A_{ij}\,
\frac{\bigl(X_j^\ell(t)\bigr)^{1/3}}{1+\bigl(X_j^\ell(t)\bigr)^{1/3}}
$ \\

Epidemic dynamics &
$\displaystyle
\frac{dX_i^\ell(t)}{dt}
= -X_i^\ell(t)
+ \sum_{j=1}^{M} A_{ij}\,\bigl(1-X_i^\ell(t)\bigr)\,X_j^\ell(t)
$ \\

Biochemical dynamics &
$\displaystyle
\frac{dX_i^\ell(t)}{dt}
= 1 - X_i^\ell(t)
- \sum_{j=1}^{M} A_{ij}\,X_i^\ell(t)\,X_j^\ell(t)
$ \\

Mutualistic dynamics &
$\displaystyle
\frac{dX_i^\ell(t)}{dt}
= X_i^\ell(t)\bigl(1-X_i^\ell(t)\bigr)
+ \sum_{j=1}^{M} A_{ij}\,X_i^\ell(t)\,
\frac{\bigl(X_j^\ell(t)\bigr)^2}{1+\bigl(X_j^\ell(t)\bigr)^2}
$ \\

Regulatory dynamics (degree-avert)&
$\displaystyle
\frac{dX^\ell_i(t)}{dt} = -X^\ell_i(t) + \sum_{j=1}^{M} A_{ij}\,\frac{(X^\ell_j(t))^{2}}{1 + (X^\ell_j(t))^{2}}
$ \\
\bottomrule
\end{tabular}
}
\end{table}

\begin{figure*}[!t]
\centering

\begin{minipage}[t]{0.24\textwidth}
\centering
\includegraphics[width=\linewidth]{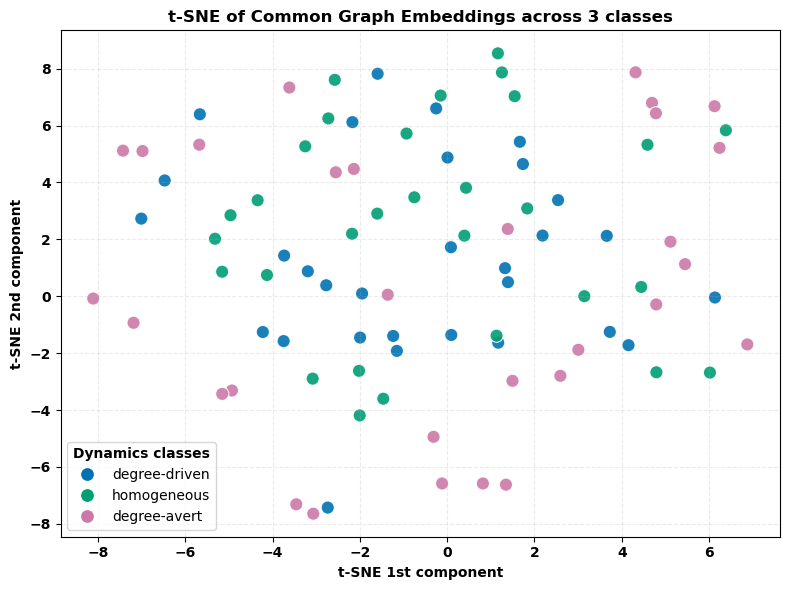}\\[-0.5ex]
{\footnotesize (a) t-SNE plot of common graph embeddings across $3$ classes}
\end{minipage}\hfill
\begin{minipage}[t]{0.24\textwidth}
\centering
\includegraphics[width=\linewidth]{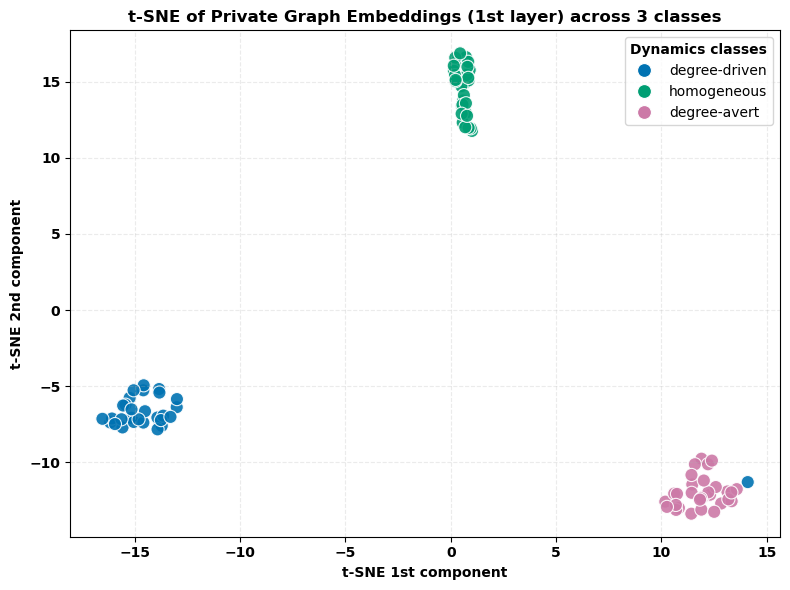}\\[-0.5ex]
{\footnotesize (b) t-SNE plot of private graph embeddings of the 1st layer %across $3$ classes
}
\end{minipage}\hfill
\begin{minipage}[t]{0.24\textwidth}
\centering
\includegraphics[width=\linewidth]{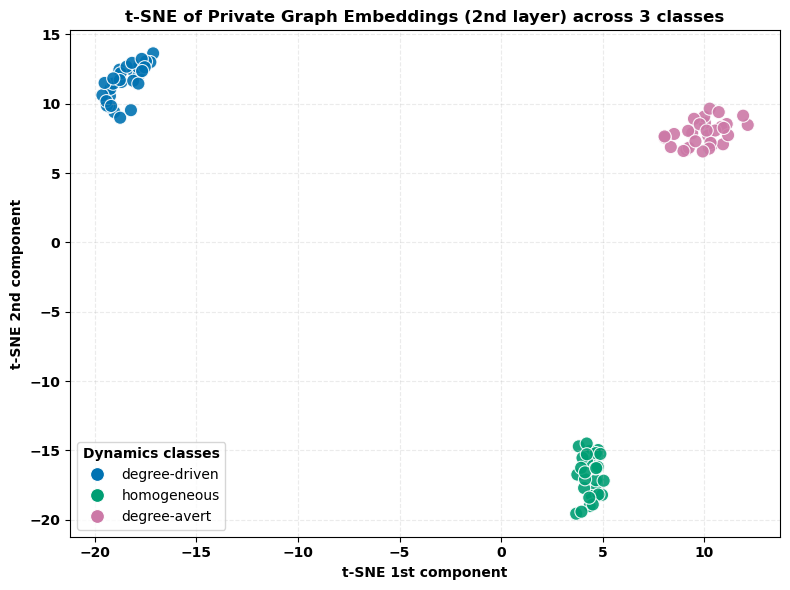}\\[-0.5ex]
{\footnotesize (c) t-SNE plot of private graph embeddings of the 2nd layer %across $3$ classes
}
\end{minipage}\hfill
\begin{minipage}[t]{0.24\textwidth}
\centering
\includegraphics[width=\linewidth]{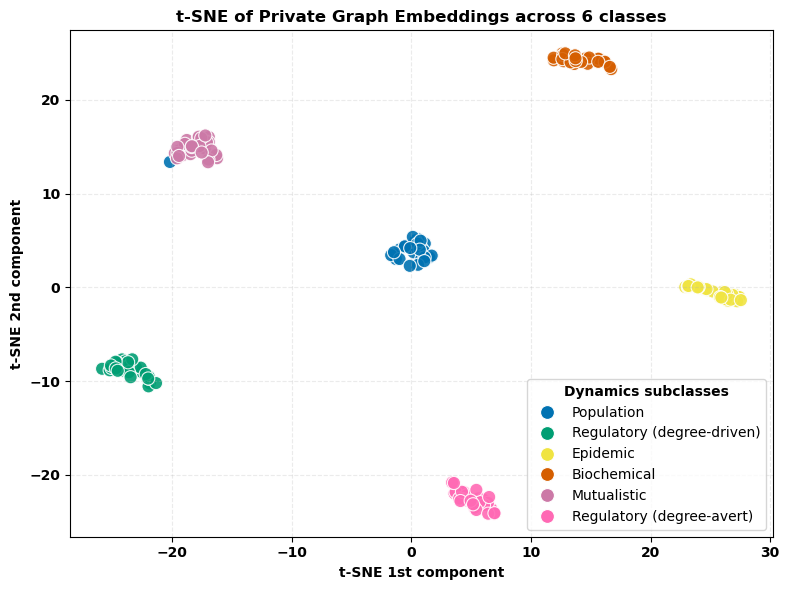}\\[-0.5ex]
{\footnotesize (d) t-SNE plot of private graph embeddings across $6$ subclasses}
\end{minipage}

\caption{t-SNE visualizations of common and private graph embeddings across the $3$ dynamical classes and $6$ subclasses.}
\label{figsyn4}

\end{figure*}

\section{Real-world datasets}
Table~\ref{tab:real_summary} summarizes the key characteristics and training configurations for all real-world multiplex datasets used in our experiments.
Learning rates (lr) and weight decays (wd) are reported as tuples: the first value is used for $\phi$ and $\psi$, and the second for the remaining network parameters. All experiments use the Adam optimizer, and the dimensions of both common and private embeddings are fixed to $8$. For all datasets except HCP, node features are row-normalized before training. Additionally, downstream node- and graph-level tasks are evaluated using a linear classifier trained for $400$ epochs with a learning rate of $0.1$.

\begin{table*}[t]
\centering
\caption{Summary of real-world multiplex datasets and experimental configurations.}
\label{tab:real_summary}
\small
\setlength{\tabcolsep}{3pt}
\renewcommand{\arraystretch}{0.8}
\begin{tabular}{lccccc}
\toprule
\textbf{Parameter} & \textbf{HCP} & \textbf{DBLP} & \textbf{IMDB} & \textbf{ACM} & \textbf{Freebase} \\
\midrule
$M$ & 360 & 7,907 & 3,550 & 3,025 & 3,492 \\
$N$ & 2 & 3 & 2 & 2 & 3 \\
Downstream task & Graph classification & Node classification & Node classification & Node classification & Node classification \\
Number of classes & 3 & 4 & 3 & 3 & 3 \\
$N_{\mathrm{aug}}$ & 20 & 20 & 20 & 20 & 30 \\
$r$ & 0.6 & 0.6 & 0.6 & 0.6 & 0.6 \\
$\sigma$ & 0.1 & 0.1 & 0.1 & 0.1 & 0.1 \\
Epochs & 200 & 400 & 400 & 400 & 400 \\
Learning rates & $(10^{-2},10^{-2})$ & $(10^{-2},10^{-4})$ & $(10^{-2},10^{-4})$ & $(10^{-2},10^{-4})$ & $(10^{-2},10^{-4})$ \\
Loss weights $(\mathcal{L}_{\text{match}}, \mathcal{L}_{\text{sup}},\mathcal{L}_{\text{causal}})$ & $(1, 0.01,0.01)$ & $(1, 0.05,0.05)$ & $(1, 0.1,0.01)$ & $(1, 0.5,0.5)$ & $(1, 0.1,0.01)$ \\
Feature dropout probability & 0 & 0.1 & 0.1 & 0.1 & 0.1 \\
Weight decay & $(10^{-4},0)$ & $(10^{-4},0)$ & $(10^{-4},0)$ & $(10^{-4},0)$ & $(10^{-4},0)$ \\
Node feature dimension & 2 & 2{,}000 & 2{,}000 & 1{,}870 & 3{,}492 \\

\bottomrule
\end{tabular}
\end{table*}

\textbf{DBLP \cite{Jing_2021}.} 
Nodes correspond to research papers, and each layer encodes a different relation among them. Papers are represented by bag-of-words features derived from their abstracts and belong to four research areas: Data Mining, Artificial Intelligence, Computer Vision, and Natural Language Processing. The multiplex graph has three layers corresponding to the Paper–Author–Paper, paper–paper reference, and Paper–Author–Term–Author–Paper relations.

\textbf{IMDB \cite{Jing_2021}.} Nodes correspond to movies, and each layer captures a different semantic relation among the same set of nodes. Each movie is associated with a bag-of-words feature vector derived from its plot. The multiplex graph has $2$ layers corresponding to the Movie–Actor–Movie and Movie–Director–Movie relations. Movies are categorized into three genre classes: Action, Comedy, and Drama.

\textbf{ACM \cite{Jing_2021}.} Nodes correspond to research papers, and each layer encodes a different paper–paper relation. Each node is associated with a bag-of-words feature vector derived from the paper abstract, and the task is to classify papers into research-area classes (database, wireless communication, and data mining). The multiplex graph has $2$ layers corresponding to the Paper–Author–Paper and Paper–Field/Subject–Paper meta-path graphs.

\textbf{Freebase \cite{pmlr-v202-mo23a}.} Nodes correspond to movies, and each layer captures a different relation between movies via meta-path–induced graphs. As no node attributes are available, we use one-hot node features. The multiplex graph has $3$ layers corresponding to the Movie–Actor–Movie, Movie–Director–Movie, and Movie–Writer–Movie relations.

\textbf{HCP \cite{VanEssen2013HCP}.}
We evaluate CaDeM on functional MRI data from the Human Connectome Project (HCP). 
The HCP dataset includes two resting-state fMRI acquisitions and seven task-based fMRI paradigms, designed to probe a broad range of cognitive functions. The time series used in our analysis are the parcel-level BOLD signals extracted from the HCP fMRI data. These signals are obtained by averaging voxel-wise BOLD activity within each of the 360 cortical regions defined by the Glasser multi-modal parcellation (MMP) atlas.
For each subject and task, we construct an FC matrix by computing the Pearson correlation between parcel-level BOLD time series. We form a $2$ layer multiplex graph by splitting FC into co-activation and anti-activation layers. Each layer is binarized via thresholding; we then 
eliminate isolated nodes by connecting each isolate to its maximum-weight neighbor.
For node features, we use the skewness and kurtosis of the BOLD time series.

\section{Computational Complexity Analysis}

In our framework, each layer is processed by two separate one-layer GCN encoders. For a sparse graph, each encoder has computational complexity $\mathcal{O}(|E^\ell|H + MFH)$, where $\ell$ denotes the layer index, $H$ is the hidden dimension of GCN and $F$ the input feature dimension; the term $|E^\ell|H$ arises from sparse neighborhood aggregation and $MFH$ from feature transformation.
Across all $N$ layers, the encoder cost per epoch is $\mathcal{O}\left(
\sum_{\ell=1}^{N}
\left(
|E^\ell| H + M F H
\right)
\right)$.
After graph propagation, linear projections map hidden representations to $d$-dimensional embeddings, adding $\mathcal{O}(MHd)$ complexity. Since typically $d \ll F$, this cost is dominated by the GCN feature transformation term $\mathcal{O}(MFH)$, yielding an overall encoder complexity of
$\mathcal{O}\left(
\sum_{\ell=1}^{N}
\left(
|E^\ell| H + M F H
\right)
\right)$.
%\end{equation}

The SVD operation in $\mathcal{L}_\text{matching}$ has complexity $\mathcal{O}(Md^2)$. Since $d$ is often small, this cost is negligible compared to the GCN propagation cost.

The computational cost of $\mathcal{L}_\text{self-supervised}$ over $N'$ augmented graph embeddings scales as
$\mathcal{O}\left(
N' (M d + d N)
\right)$,
where $Md$ arises from global pooling over $M$ node embeddings of dimension $d$, and $dN$ corresponds to $\phi$ mapping $d$-dimensional graph representations to $N$ output logits.

For $\mathcal{L}_{\mathrm{causal}}$, we approximate the expectation over the common embeddings by randomly sampling a limited number $K$ of pairings per private embedding. This Monte Carlo sampling strategy \cite{rubinstein2017simulation}, commonly used in contrastive and self-supervised learning, is sufficient to enforce the desired invariance. Each sampled pairing incurs pooling and linear classification costs, yielding an effective $\mathcal{L}_{\text{causal}}$ complexity of
$\mathcal{O}\left(
N' \cdot K \cdot (M d + d N)
\right)$.
Because $K$ is a constant independent of $N'$, the complexity scales linearly in $N'$. 

The costs of $\mathcal{L}_{\text{self-supervised}}$ and $\mathcal{L}_{\text{causal}}$ are lower-order than the encoder complexity, as they operate in the low-dimensional embedding space $d$, whereas the encoder performs sparse message passing in the higher-dimensional space $H$ and scales with the number of edges $|E^\ell|$. Moreover, $N'$ is small, and $K$ is fixed and independent of $N'$. 

As a result, the total training complexity per epoch is:
$
\mathcal{O}\left(
\sum_{\ell=1}^{N}
\left(
|E^\ell| H + M F H
\right)
\right),
$
and all additional operations are lower-order terms. In practice, for most real-world graph learning tasks, either $|E^\ell| \gg M$ or $F$ is moderate, and the proposed framework does not incur any additional asymptotic computational overhead beyond that of standard sparse GNN message passing. Consequently, CaDeM preserves the scalability of standard sparse GNN models while introducing only lightweight overhead for disentangled representation learning.

\section{Effect of Graph Augmentation}

Table \ref{graph_aug} reports the performance of our framework with and without graph augmentation across 4 real-world datasets and Syn1 for the node classification task. The results consistently demonstrate that augmentation improves both Macro-F1 and Micro-F1 scores across all datasets. These results demonstrate the effectiveness of the proposed graph augmentation strategy. Since $\mathcal{L}_{\text{self-supervised}}$ and $\mathcal{L}_{\text{causal}}$ approximate expectations over common and private embeddings, multiple samples from the same underlying distribution are required. Graph augmentation generates stochastic views of each layer while preserving semantic structure, enabling more reliable estimation of these expectations. It also stabilizes optimization by increasing the effective sample size used in the loss, reducing gradient variance, and producing smoother training.

\begin{table}[t]
\centering
\caption{Node classification performance of CaDeM with and without graph augmentation.}
\resizebox{\columnwidth}{!}{
\begin{tabular}{lcccc}
\toprule
& \multicolumn{2}{c}{\textbf{Without augmentation}} 
& \multicolumn{2}{c}{\textbf{With augmentation}} \\
\cmidrule(lr){2-3} \cmidrule(lr){4-5}
\textbf{Dataset} & \textbf{Macro F1} & \textbf{Micro F1} 
& \textbf{Macro F1} & \textbf{Micro F1} \\
\midrule
IMDB & $0.6571 \pm 0.0172$ & $0.6588 \pm 0.0159$ 
     & $\mathbf{0.6667 \pm 0.0158}$ & $\mathbf{0.6682 \pm 0.0155}$ \\

ACM & $0.9173 \pm 0.0100$ & $0.9167 \pm 0.0100$ 
    & $\mathbf{0.9287 \pm 0.0084}$ & $\mathbf{0.9283 \pm 0.0085}$ \\

DBLP & $0.8339 \pm 0.0158$ & $0.8255 \pm 0.0143$ 
     & $\mathbf{0.8497 \pm 0.0111}$ & $\mathbf{0.8414 \pm 0.0108}$ \\

Freebase & $0.6122 \pm 0.0126$ & $0.6766 \pm 0.0098$ 
         & $\mathbf{0.6224 \pm 0.0121}$ & $\mathbf{0.6887 \pm 0.0137}$ \\

First synthetic data & $0.7843 \pm 0.0654$ & $0.7900 \pm 0.0678$ 
                     & $\mathbf{0.8178 \pm 0.0516}$ & $\mathbf{0.8200 \pm 0.0510}$ \\
\bottomrule
\end{tabular}
}
\label{graph_aug}

\end{table}

\end{document}